\documentclass{article}

\usepackage[final,nonatbib]{neurips_2022}
\usepackage[numbers]{natbib}

\usepackage[utf8]{inputenc} %
\usepackage[T1]{fontenc}    %
\usepackage{hyperref}       %
\usepackage{url}            %
\usepackage{booktabs}       %
\usepackage{amsfonts}       %
\usepackage{nicefrac}       %
\usepackage{microtype}      %
\usepackage{xcolor}         %
\usepackage{graphicx}
\usepackage{subfigure}

\usepackage{amsmath}
\usepackage{amssymb}
\usepackage{mathtools}
\usepackage{amsthm}
\usepackage{enumitem}

\usepackage{breqn}
\usepackage{colortbl}
\usepackage{fancyhdr}
\usepackage{comment}
\usepackage{bm}
\usepackage{multirow}
\usepackage{algorithm,algorithmicx,algpseudocode}

\usepackage[capitalize,noabbrev]{cleveref}

\theoremstyle{plain}
\newtheorem{theorem}{Theorem}[section]

\newtheorem{lemma}[theorem]{Lemma}
\newtheorem{corollary}[theorem]{Corollary}
\theoremstyle{definition}

\theoremstyle{remark}

\DeclareMathOperator*{\argmax}{arg\,max}
\DeclareMathOperator*{\argmin}{arg\,min}
\DeclareMathOperator*{\aug}{aug}

\newcommand{\figref}[1]{Fig.~\ref{#1}}
\newcommand{\tabref}[1]{Tab.~\ref{#1}}
\newcommand{\secref}[1]{Sec.~\ref{#1}}

\newcommand{\beq}{\begin{equation}}

\newcommand{\eeq}{\end{equation}}

\newcommand{\Lc}{{\cal{L}}}

\newcommand{\Jc}{{\cal{J}}}

\newcommand{\x}{\vct{x}}

\newcommand{\W}{\mtx{W}}

\definecolor{emmanuel}{RGB}{255,127,0}

\newcommand{\R}{\mathbb{R}}

\newcommand{\E}{\operatorname{\mathbb{E}}}

\newcommand{\vct}[1]{\bm{#1}}
\newcommand{\mtx}[1]{\bm{#1}}

\newcommand{\X}{{\mtx{X}}}

\def \endprf{\hfill {\vrule height6pt width6pt depth0pt}\medskip}

\title{Data-Efficient Augmentation for Training Neural Networks}

\author{%
  Tian Yu Liu \\
  Department of Computer Science\\
  University of California, Los Angeles\\
  \texttt{tianyu@cs.ucla.edu} \\
   \And
  Baharan Mirzasoleiman \\
  Department of Computer Science\\
  University of California, Los Angeles\\
  \texttt{baharan@cs.ucla.edu} \\ 
}

\begin{document}
\maketitle

\begin{abstract}
 Data augmentation is essential to achieve state-of-the-art performance in many deep learning applications. However, the most effective augmentation techniques become computationally prohibitive for even medium-sized datasets.
To address this, we propose a rigorous technique to select subsets of data points that when augmented, closely capture the training dynamics of full data augmentation.
We first show that data augmentation, modeled as additive perturbations, improves learning and generalization by relatively enlarging and perturbing the smaller singular values of the network Jacobian, while preserving its prominent directions.
This prevents overfitting and enhances learning the harder to learn information.
Then, we propose a framework to iteratively extract small subsets of training data that when augmented, closely capture the alignment of the fully augmented Jacobian with labels/residuals. 
We prove that stochastic gradient descent applied to the augmented subsets found by our approach has similar training dynamics to that of fully augmented data.
{Our experiments demonstrate that our method achieves 6.3x speedup on CIFAR10 and 2.2x speedup on SVHN, and outperforms %
the baselines by up to 10\% across various subset sizes.}
Similarly, on TinyImageNet and ImageNet, our method beats the baselines by up to 8\%, while achieving up to 3.3x speedup across various subset sizes.
Finally, training on and augmenting 50\% subsets using our method on a version of CIFAR10 corrupted with label noise even outperforms using the full dataset. \footnote{Our code can be found at \href{https://github.com/tianyu139/data-efficient-augmentation}{https://github.com/tianyu139/data-efficient-augmentation}}
\end{abstract}

\section{Introduction}

Standard (weak) data augmentation %
transforms the training examples with e.g. rotations or crops for images, and trains on the transformed examples \textit{in place of} the original training data. While weak augmentation is effective and computationally inexpensive, 
\textit{strong}
data augmentation (in addition to weak augmentation)
is a key component in achieving nearly all state-of-the-art results in deep learning applications
\citep{shorten2019survey}.
However, %
strong data augmentation techniques often increase the training time by orders of magnitude.
First, they often have a very expensive pipeline to find or generate more complex
transformations
that best improves generalization
\cite{bowles2018gansfer,jackson2019style,lemley2017smart,wu2020generalization}. Second, \textit{appending transformed examples} to the training data is often much more effective than training on the (strongly or weakly) transformed examples \textit{in-place} of the original data. For example, appending \textit{one} transformed example to the training data is often much more effective than training on \textit{two} transformed examples \textit{in place} of every original training data, %
while both strategies have the same computational cost
(\textit{c.f.} Appendix \ref{sec:append}).
Hence, to obtain the state-of-the-art performance, multiple %
augmented examples are added for every single data point and to each training iteration \cite{hendrycks2019augmix,wu2020generalization}. In this case, even if producing transformations are cheap, such methods increases the size of the training data by orders of magnitude.
As a result,
state-of-the-art data augmentation techniques %
become computationally prohibitive for even medium-sized real-world problems. %
For example, the state-of-the-art augmentation of \citep{wu2020generalization}, which appends every example with its highest-loss transformations, increases the training time of ResNet20 on CIFAR10 by 13x on an Nvidia A40 GPU (\textit{c.f.} Sec. \ref{sec:experiments}).

To make state-of-the-art data augmentation more efficient and scalable, %
an effective approach is to 
carefully select a small subset of the training data such that augmenting only the subset
provides similar training dynamics to that of full data augmentation.
If such a subset can be quickly found, it would directly lead to a significant reduction in
storage and training costs.
First, while standard in-place augmentation can be applied to the entire data, the strong and \textit{expensive transformations} can be only produced for the examples in the subset. Besides, only the transformed elements of the subset can be \textit{appended} to the training data. Finally, when the data is larger than the training budget, one can train on random subsets (with standard in-place augmentation) and augment coresets (by strong augmentation and/or appending transformations) to achieve a superior performance.

Despite the efficiency and scalability that it can provide, this direction has remained largely unexplored.
Existing studies are limited to 
fully training a network and 
subsampling data points based on their loss or influence, 
for augmentation in subsequent training runs
\citep{kuchnik2018efficient}.
However, this method is prohibitive for large datasets, provides a marginal improvement over augmenting random subsets, %
and does not provide any theoretical guarantee for the performance of the network trained on the augmented subsets. %
Besides, when the data contains mislabeled examples, augmentation methods that
select examples with maximum loss, and
append their transformed versions to the data,
degrade the performance by selecting and appending
several noisy labels.

A major challenge in finding the most effective data points for augmentation is %
to theoretically %
understand how data augmentation affects the optimization and generalization of neural networks.
Existing theoretical results are %
mainly limited to simple linear classifiers %
and analyze data augmentation as enlarging the span of the training data \citep{wu2020generalization}, providing a regularization effect \citep{bishop1995training,dao2019kernel,wager2013dropout,wu2020generalization}, enlarging the margin of a linear classifier  \citep{rajput2019does}, or 
having a variance reduction effect %
\citep{chen2019invariance}.
However, such tools do not provide insights on the effect of data augmentation on training %
deep neural networks. %

Here, we study the effect of label invariant data augmentation %
on training dynamics of overparameterized neural networks.
Theoretically, we model data augmentation by \textit{bounded additive perturbations} \citep{rajput2019does}, and
analyze its effect on neural network Jacobian matrix containing all its first-order partial derivatives \citep{arora2019fine}. 
We show that label invariant additive data augmentation %
\textit{proportionally} enlarges but more importantly \textit{perturbs} the singular values of the Jacobian, particularly the smaller ones, while maintaining prominent directions of the Jacobian. In doing so, data augmentation \textit{regularizes} training by adding bounded but varying perturbations to the gradients. In addition, it \textit{speeds up} learning harder to learn information. Thus, it prevents overfitting and improves generalization.
Empirically, we show that %
the same effect can be observed
for various strong augmentations, e.g., AutoAugment \cite{cubuk2019autoaugment}, CutOut \cite{devries2017improved}, and AugMix \cite{hendrycks2019augmix}.\footnote{ 
We note that our results are in line with that of \citep{shen2022data}, that in parallel to our work, analyzed the effect of linear transformations on a two-layer convolutional network, and showed that it can make the hard to learn features more likely to be captured during training.}

Next, we develop a rigorous method 
to iteratively find small weighted subsets (coresets) that 
when augmented, closely capture the 
alignment between the Jacobian of the full augmented data with the label/residual vector.
We show that the most effective subsets for data augmentation are the set of examples 
that when data is mapped to the gradient space, have %
the most centrally located gradients. 
This problem can be %
formulated as maximizing a submodular function. %
The subsets can be efficiently extracted using a fast greedy algorithm which operates on small dimensional gradient proxies, with only a small additional cost.
We prove that augmenting the coresets guarantees similar training dynamics to that of full data augmentation.
We also show that augmenting our coresets achieve a superior accuracy in presence of noisy labeled examples.

We demonstrate the effectiveness of our approach applied to CIFAR10 (ResNet20, WideResNet-28-10), CIFAR10-IB (ResNet32), SVHN (ResNet32), noisy-CIFAR10 (ResNet20), Caltech256 (ResNet18, ResNet50), TinyImageNet (ResNet50), and ImageNet (ResNet50) compared to random and max-loss baselines \citep{kuchnik2018efficient}. 
We show the effectiveness of our approach (in presence of standard augmentation) in the following cases:
\begin{itemize}[leftmargin=3.mm]
   \vspace{-2mm}
    \item \textbf{When producing augmentations is expensive and/or they are appended to the training data:} 

We show that for the state-of-the-art augmentation method of \cite{wu2020generalization} applied to CIFAR10/ResNet20 it is 3.43x faster to train on the whole dataset and only augment our coresets of size 30\%, compared to training and augmenting the whole dataset. At the same time, we achieve 75\% of the accuracy improvement of training on and augmenting the full data with the method of \cite{wu2020generalization}, outperforming both max-loss and random baselines by up to 10\%.  

\vspace{-1mm}
\item \textbf{When data is larger than the training budget:} 
We show that we can achieve 71.99\% test accuracy on ResNet50/ImageNet when training on and augmenting only 30\% subsets for 90 epochs. Compared to AutoAugment \cite{cubuk2019autoaugment}, despite using only 30\% subsets, we achieve 92.8\% of the original reported accuracy while boasting 5x speedup in the training time. 
Similarly, on Caltech256/ResNet18, training on and augmenting 10\% coresets with AutoAugment yields 65.4\% accuracy, improving over random 10\% subsets by 5.8\% and over only weak augmentation by 17.4\%. \looseness=-1

\vspace{-1mm}
\item \textbf{When data contains mislabeled examples:} 
We show that training on and strongly augmenting 50\% subsets using our method on CIFAR10 with 50\% noisy labels achieves 76.20\% test accuracy. Notably, this %
yields a superior performance to
training on and strongly augmenting the full data.\looseness=-2

\end{itemize}

\section{Additional Related Work}
Strong data augmentation methods achieve state-of-the-art performance by finding the set of transformations for every example that best improves the performance. %
Methods like AutoAugment \citep{cubuk2019autoaugment}, RandAugment \cite{cubuk2020randaugment}, and Faster RandAugment  \citep{cubuk2020randaugment} search over a (possibly large) space of transformations to find sequences of transformations that best improve generalization \citep{cubuk2019autoaugment,cubuk2020randaugment,luo2020data,wu2020generalization}. %
Other techniques involve a very expensive pipeline for generating the transformations. For example, some use Generative Adversarial Networks to directly learn new transformations \citep{baluja2017adversarial,luo2020data,mirza2014conditional,ratner2017learning}. 
Strong augmentations like Smart Augmentation \cite{lemley2017smart}, Neural Style Transfer-based \cite{jackson2019style}, and GAN-based augmentations \cite{bowles2018gansfer} require an expensive forward pass through a deep network for input transformations. 
For example, \cite{jackson2019style} increases training time by 2.8x for training ResNet18 on Caltech256.
Similarly, \cite{wu2020generalization} generates multiple augmentations for each training example, and selects the ones with the highest loss.\looseness=-1

Strong data augmentation methods either replace the original example by its transformed version, or append the generated transformations to the training data. Crucially, appending the training data with transformations is much more effective in improving the generalization performance. 
Hence, the most effective data augmentation methods such as that of \cite{wu2020generalization} and AugMix \cite{hendrycks2019augmix} append the transformed examples to the training data. %
In Appendix \ref{sec:append}, we show that even for cheaper strong augmentation methods such as AutoAugment \cite{cubuk2019autoaugment}, while replacing the original training examples with transformations may decrease the performance, appending the augmentations significantly improves the performance.
Appending the training data with augmentations, however, increase the training time by orders of magnitude. For example,
AugMix \cite{hendrycks2019augmix} that outperforms AutoAugment increases the training time by at least 3x by appending extra augmented examples, and \cite{wu2020generalization} increases training time by 13x due to appending and forwarding additional augmented examples through the model.

\section{Problem Formulation}\label{sec:formulation}
We begin by 
formally describing the problem of learning from augmented data.
Consider a dataset 
$\mathcal{D}_{train} = (\bm{X}_{train}, \bm{y}_{train})$,
where $\bm{X}_{train} = (\bm{x}_1,\cdots, \bm{x}_n) \in \mathbb{R}^{d\times n}$ is the set of $n$ normalized data points $\bm{x}_i \in [0, 1]^d$, from the index set $V$, %
and $\bm{y}_{train}=(y_1,\!\cdots\!,y_n) 
\in \{ y \in 
\{\nu_1,\nu_2, \cdots, \nu_C\}\}$ with $\{\nu_j\}_{j=1}^C \in [0,1]$.

\textbf{The additive perturbation model.} 
Following \citep{rajput2019does} we model data augmentation as an arbitrary bounded additive perturbation $\bm{\epsilon}$, with $\|\bm{\epsilon}\| \leq \epsilon_0$. For a given $\epsilon_0$ and the set of all possible transformations $\mathcal{A}$, %
we study the transformations selected from
$\mathcal{S} \subseteq \mathcal{A}$ satisfying
\begin{align}\label{eq:augment}
    \mathcal{S} = \{ T_i \in \mathcal{A} \ | \ \|T_i( \bm{x}) - \bm{x}\| \leq \epsilon_0 \ \forall  \bm{x} \in \X^{train} \}.
\end{align}
While the additive perturbation model cannot represent all augmentations,
most real-world augmentations are bounded to preserve the regularities of natural images (e.g. AutoAugment \cite{cubuk2019autoaugment} finds that a 6 degree rotation is optimal for CIFAR10). 
Thus, under local smoothness of images, additive perturbation can model bounded transformations such as small rotations, crops, shearing, and pixel-wise transformations like sharpening, blurring, color distortions, structured adversarial perturbation \cite{luo2020data}. 
As such, we see the effects of additive augmentation on the singular spectrum holds even under real-world augmentation settings (\textit{c.f.} \figref{fig:singular-values-and-vectors-large} in the Appendix). %
However, this model is indeed limited when applied to augmentations that cannot be reduced to perturbations, such as horizontal/vertical flips and large translations. 
We extend our theoretical analysis to augmentations modeled as arbitrary linear transforms (e.g. as mentioned, horizontal flips) in \ref{app:aug-linear-transform}.

The set of augmentations at iteration $t$ generating $r$ augmented examples per data point can be specified, with abuse of notation, as
$\mathcal{D}_{aug}^{t}=\{\bigcup^r_{i=1} \big(T_i^{t}(\X_{train}),\bm{y}_{train}\big)\}$, where $|\mathcal{D}_{aug}^{t}| = rn$ and $T_i^{t}(\X_{train})$ transforms all the training data points with the set of transformations $T_i^{t}\subset\mathcal{S}$ at iteration $t$.
We denote $\X_{aug}^{t}=\{\bigcup^r_{i=1}T_i^{t}(\X_{train})\}$ and $\bm{y}_{aug}^{t}=\{\bigcup^r_{i=1}\bm{y}_{train}\}$.

\textbf{Training on the augmented data.} Let $f(\bm{W}, \bm{x})$ be an arbitrary neural network with $m$ vectorized (trainable) parameters
$\bm{W} \!\!\in\!\! \mathbb{R}^{m}$.
We assume that the network is trained using (stochastic) gradient descent with learning rate $\eta$ to minimize the squared loss $\mathcal{L}$ over the original and augmented %
training examples $\mathcal{D}^{t}=\{\mathcal{D}_{train}\cup\mathcal{D}_{aug}^{t}\}$ with associated index set $V^t$, at every iteration $t$. I.e., %
\begin{equation}
    \mathcal{L}(\bm{W}^t,\X) := \frac{1}{2} \sum_{i \in V^t} \mathcal{L}_i(\bm{W}^t,\bm{x}_i) %
    := \frac{1}{2} \sum_{\substack{(\bm{x}_i,y_i)
    \in\mathcal{D}^{t}}}%
    \| f(\bm{W}^t, \bm{x}_i) - {y}_i\|_2^2. %
\end{equation}
The gradient update at iteration $t$ is given by 
\begin{align}\label{eq:gd}
    &\bm{W}^{t + 1} = \bm{W}^{t} - \eta \nabla \mathcal{L} (\bm{W}^t, \bm{X}), \quad \text{s.t.}\quad
    &\nabla \mathcal{L} (\bm{W}^t, \bm{X}) = \Jc^T(\bm{W}^t,\X) (f(\bm{W}^t, \bm{X}) - \bm{y}),
\end{align}
where $\bm{X}^{t}=\{\X_{train}\cup\X_{aug}^{t}\}$ and $\bm{y}^t=\{\bm{y}_{train}\cup\bm{y}_{aug}^t\}$ are the set of original and augmented examples and their labels,
$\Jc(\W,\X) \in \mathbb{R}^{n\times m}$ is the  Jacobian matrix associated with $f$, and $\bm{r}^t=f(\bm{W}^t, \bm{X}) - \bm{y}$ is the residual. %

We further assume that $\Jc$ is smooth
with Lipschitz constant $L$. I.e.,
    $\| \mathcal{J}(\bm{W}, \bm{x}_i) \!\!-\!\! \mathcal{J}(\bm{W}, \bm{x}_j) \| \leq L \|\ \bm{x}_i - \bm{x}_j \| ~~ \forall \ \bm{x}_i,\bm{x}_j \in \bm{X}^{}.$
Thus,
for any transformation $T_j\in\mathcal{S}$, we have $\| \mathcal{J}(\bm{W}, \bm{x}_i) \!-\! \mathcal{J}(\bm{W}, T_j(\bm{x}_i)) \| \leq L \epsilon_0$. %
Finally, denoting $\Jc\!=\!\Jc(\W\!,\!\X_{train})$ and $\tilde{\Jc}=\Jc(\W\!,\!T_j(\X_{train}))$,\!
we get  $\!\tilde{\Jc}\!\!=\!\!\Jc\!+\bm{E}$, where %
$\bm{E}$ is the perturbation matrix with $\!\|\bm{E}\|_2\leq\!\|\bm{E}\|_F\leq\! \sqrt{n}L\epsilon_0$.

\section{Data Augmentation Improves Learning}
\label{sec:aug_general}

In this section, we analyze the effect of data augmentation on 
training dynamics of neural networks, 
and show that data augmentation can provably prevent overfitting.
To do so, 
we leverage the recent %
results that characterize the training
dynamics based on properties of neural network Jacobian and the corresponding
Neural Tangent Kernel (NTK) \citep{jacot2018neural}
defined as $\bm{\Theta}=\Jc(\bm{W},\bm{X})\Jc(\bm{W},\bm{X})^T$.
Formally: 
\vspace{-2mm}
\begin{align}\label{eq:shrinkage}
    \bm{r}^t \!\!=\!\! \sum_{i=1}^n (1-\eta\lambda_i)(\bm{u}_i \bm{u}_i^T) \bm{r}^{t-1} \!\!= \sum_{i=1}^n (1-\eta\lambda_i)^t(\bm{u}_i \bm{u}_i^T) \bm{r}^{0}
     ,
    \vspace{-3mm}
\end{align}
where $\bm{\Theta} = \bm{U}\Lambda \bm{U}^T =\sum_{i=1} \lambda_i \bm{u}_i\bm{u}_i^T$
is the eigendecomposition of the NTK \citep{arora2019fine}. %
Although the constant NTK assumption holds only in the infinite width limit, \cite{lee2019wide} found close empirical agreement between the NTK dynamics and the true dynamics for wide but practical networks, such as wide ResNet architectures \citep{zagoruyko2016wide}.
Eq. \eqref{eq:shrinkage} shows that the
training dynamics depend on the alignment of the NTK with the residual vector at every iteration $t$. 
Next, we prove that for small perturbations $\epsilon_0$, data augmentation %
prevents overfitting and improves
generalization by proportionally enlarging and perturbing smaller eigenvalues of the NTK relatively more, while preserving its prominent directions.

\subsection{Effect of Augmentation on  Eigenvalues of the NTK}\label{sec:singularvalues}
We first investigate the effect of data augmentation %
on the singular values of the Jacobian, and use this result to bound the change in the eigenvalues of the NTK. 
To characterize the effect of data augmentation on singular values of the perturbed Jacobian $\tilde{\Jc}$, we rely on Weyl's theorem \citep{weyl1912asymptotic} stating that under bounded perturbations $\bm{E}$, no singular value can move more than the norm of the perturbations. Formally, %
$|\tilde{\sigma}_i - \sigma_i |\leq \|\bm{E}\|_2$, where $\tilde{\sigma}_i$ and $\sigma_i$ are the singular values of the perturbed and original Jacobian respectively. %
Crucially, data augmentation affects larger and smaller singular values differently. 
Let $\bm{P}$ be orthogonal projection onto the column space of $\Jc^T$, and $\bm{P}_{\perp} = \bm{I} - \bm{P}$ be the projection onto its orthogonal complement subspace. Then, the singular values of the perturbed Jacobian $\tilde{\Jc}^T$ are
$\tilde{\sigma}_i^2 = (\sigma_i + \mu_i)^2 + \zeta_i^2$,
where
$|\mu_i| \leq \|\bm{PE}\|_2$, and
$\sigma_{\min}(\bm{P}_{\perp}\bm{E}) \leq \zeta_i \leq \|\bm{P}_{\perp}\bm{E}\|_2$, $\sigma_{\min}$ the smallest singular value of $\Jc^T$ \citep{ %
stewart1979note}.
Since the eigenvalues of the projection matrix $\bm{P}$ are either 0 or 1, as the number of dimensions $m$ grows, for bounded perturbations we get that on average $\mu_i^2=\mathcal{O}(1)$ and $\zeta_i^2=\mathcal{O}(m)$.
Thus, the second term dominates and  increase of small singular values under perturbation is proportional to $\sqrt{m}$. However, for larger singular values, first term dominates and hence $\tilde{\sigma}_i - \sigma_i \cong \mu_i$.
Thus
in general, small singular values can become
proportionally larger, while larger singular values 
remain relatively unchanged. %
The following Lemma characterizes the \textit{expected} change to the eignvalues of the NTK.
\begin{lemma}
\label{lemma:singular-value-perturbation-bounds}
Data augmentation as additive perturbations bounded by small $\epsilon_0$ %
results in the following expected change to the eigenvalues of the NTK:
\begin{equation}
    \E[\tilde{\lambda}_i] = \E[\tilde{\sigma}^2_i] = \sigma_i^2 + {\sigma_i}(1-2p_i)\|\bm{E}\| +  \|\bm{E}\|^2/3
\end{equation}
where $p_i := \mathbb{P}(\tilde{\sigma_i} - \sigma_i < 0)$ is the probability that $\sigma_i$ decreases as a result of data augmentation, and is smaller for smaller singular values.
\vspace{-2mm}
\end{lemma}
The proof can be found in Appendix \ref{appendix:proof-singular-value-perturbation-bounds}. 

Next, we discuss the effect of data augmentation on singular vectors of the Jacobian and show that it mainly affects the non-prominent directions of the Jacobian spectrum, but to a smaller extent compared to the singular values.

\begin{figure*}[t]
    \centering
    \begin{subfigure}[MNIST - $\Delta \sigma_i$ 
    \label{fig:mnist_noise8_and_16_relative_values}]{
    \includegraphics[width=.225\textwidth,trim=10mm 0 0mm 10mm]{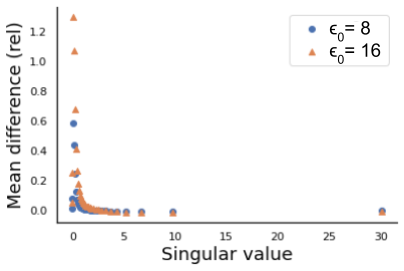}}
    \end{subfigure}\hspace{0mm}
    \begin{subfigure}[\!CIFAR - $\Delta \sigma_i$ %
    \label{fig:cifar_noise8_and_16_relative_values}]{
    \includegraphics[width=.225\textwidth,trim=10mm 0 0mm 10mm]{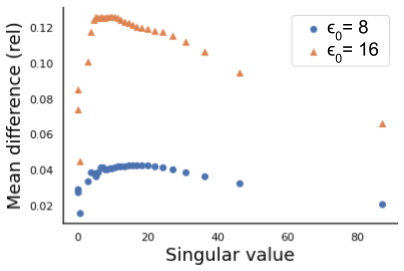}}
    \end{subfigure}\hspace{0mm}
    \begin{subfigure}[\!MNIST - Subs. Angle \label{fig:mnist_noise16_vectorspace}]{
    \includegraphics[width=.225\textwidth,trim=10mm 0 0mm 10mm]{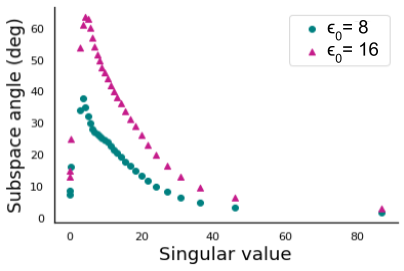}}
    \end{subfigure}\hspace{0mm}
    \begin{subfigure}[\!CIFAR - Sub. Angle \label{fig:cifar_noise16_vectorspace}]{
    \includegraphics[width=.225\textwidth,trim=10mm 0 0mm 10mm]{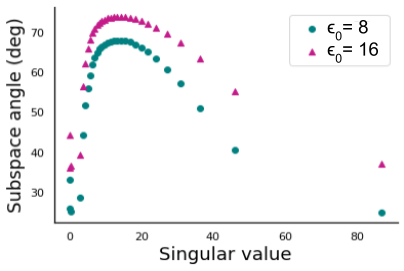}}
    \end{subfigure}
    \vspace{-2mm}
    \caption{Effect of augmentations on the singular spectrum of the network Jacobian of ResNet20 trained on CIFAR10, and a MLP on MNIST, trained till epoch 15.  (a), (b) Difference in singular values and (c), (d) singular subspace angles between the original and augmented data with bounded perturbations with $\epsilon_0 = 8$ and $\epsilon_0 = 16$ for different ranges of singular values. {Note that augmentations with larger bound $\epsilon_0$ results in larger perturbations to the singular spectrum.}
    }\vspace{-2mm}
    \label{fig:spectrum}
\end{figure*}

\subsection{\!Effect of Augmentation on Eigenvectors of the NTK}
\label{sec:singularvectors}
Here, we focus on characterizing the effect of data augmentation on the eigenspace of the NTK. 
Let the singular subspace decomposition of the Jacobian be $\Jc=\bm{U} \bm{\Sigma} \bm{V}^T$. Then for the NTK, we have $\bm{\Theta}=\Jc \Jc^T=\bm{U \Sigma V}^T \bm{V \Sigma U}^T= \bm{U}\bm{\Sigma}^2 \bm{U}^T$ (since $\bm{V}^T\bm{V}=\bm{I}$). Hence, the perturbation of the eigenspace of the NTK is the same as perturbation of the left singular subspace of the Jacobian $\Jc$.
Suppose $\sigma_i$ are singular values of the Jacobian. 
Let the perturbed Jacobian be $\tilde{\Jc}=\Jc+\bm{E}$, and
denote the eigengap
$\gamma_0 = \min\{\sigma_i - \sigma_{i+1} : i = 1, \cdots , r\}$ where $\sigma_{r+1} := 0$.
Assuming $\gamma_0 \geq 2\|\bm{E}\|_2$, a
combination of Wedin’s theorem \citep{wedin1972perturbation} and Mirsky’s inequality \citep{mirsky1960symmetric} %
implies  
\begin{equation}\label{eq:eigenvector}
    \|\bm{u}_i-\tilde{\bm{u}}_i\|\leq 2\sqrt{2}\|\bm{E}\|/\gamma_0. %
\end{equation}
This result provides an upper-bound on the change of every left singular vectors of the Jacobian. 

However as we discuss below, data augmentation affects larger and smaller singular directions differently.
To see the effect of data augmentation on every singular vectors of the Jacobian, let the subspace decomposition of Jacobian be $\Jc=\bm{U\Sigma}\bm{V}^T=\bm{U}_s\bm{\Sigma}_s \bm{V}^T_s+\bm{U}_n\bm{\Sigma}_n \bm{V}^T_n$ , where $\bm{U}_s$ associated with nonzero singular values, spans the column space of $\Jc$, which is also called the signal subspace, and $\bm{U}_n$, associated with zero singular values ($\bm{\Sigma}_n=0$), spans the orthogonal space of $\bm{U}_s$, which is also called the noise subspace. 
Similarly, let the subspace decomposition of the perturbed Jacobian be $\tilde{\Jc}=\tilde{\bm{U}}\tilde{\bm{\Sigma}} \tilde{\bm{V}}^T=\tilde{\bm{U}}_s\tilde{\bm{\Sigma}_s} \tilde{\bm{V}}^T_s+\tilde{\bm{U}}_n\tilde{\bm{\Sigma}_n} \tilde{\bm{V}}^T_n$, %
and $\tilde{\bm{U}}_s=\bm{U}_s+\Delta \bm{U}_s$, where $\Delta \bm{U}_s$ is the perturbation of the singular vectors that span the signal subspace.
Then the following general first-order expression for the perturbation of the orthogonal
subspace due to perturbations of the Jacobian characterize the change of the singular directions:
$\Delta \bm{U}_s=\bm{U}_n \bm{U}_n^T \bm{E} \bm{V}_s \bm{\Sigma}_s^{-1}$ \citep{li1993performance}. We see that %
singular vectors associated to larger singular values are more robust to data augmentation, compared to
others.
Note that {in general} singular vectors are more robust %
than singular values. 

\figref{fig:spectrum} shows the effect of %
{perturbations} with $\epsilon_0= 8, 16$ on singular values and singular vectors of the Jacobian matrix for a 1 hidden layer MLP trained on MNIST, %
and ResNet20 trained on CIFAR10. As calculating the entire Jacobian spectrum is computationally prohibitive, data is subsampled from 3 classes. 
We report %
the effect of other real-world augmentation techniques, such as random crops, flips, rotations and Autoaugment \cite{cubuk2019autoaugment} - which includes %
translations, contrast, and brightness transforms - %
in Appendix C. 
We observe that data augmentation increases smaller singular values relatively more. 
On the other hand, it affects prominent singular vectors of the Jacobian to a smaller extent. %

\vspace{-1.5mm}
\subsection{\!Augmentation Improves Training \& Generalization}
\vspace{-1.5mm}

Recent studies have revealed that the Jacobian matrix of common neural networks is low rank. That is there are a number of large singular values and the rest of the singular values are small. Based on this, the Jacobian spectrum can be divided into information and nuisance spaces \cite{oymak2019generalization}. Information space is a lower dimensional space associated with the prominent singular value/vectors of the Jacobian. Nuisance space is a high dimensional space corresponding to smaller singular value/vectors of the Jacobian. While learning over information space is fast and generalizes well, learning over nuisance space is slow and results in overfitting \cite{oymak2019generalization}. Importantly, recent theoretical studies connected the generalization performance to small singular values (of the information space) \cite{arora2019fine}. 

Our results show that label-preserving additive perturbations relatively enlarge the smaller singular values of the Jacobian in a \textit{stochastic} way and with a high probability. This benefits generalization in 2 ways. First, this stochastic behavior prevents overfitting along any particular singular direction \textit{in the nuisance space}, as stochastic perturbation of the \textit{smallest} singular values results in a stochastic noise to be added to the gradient at every training iteration. This prevents overfitting (thus a larger training loss as shown in {Appendix \ref{appendix:training-dynamics-vs-generalization}}), and improves generalization \citep{cubuk2020randaugment,dao2019kernel}. Theorem \ref{app:full-data-augment-dynamics}  in the Appendix characterizes the expected training dynamics resulted by data augmentation.
Second, additive perturbations improve the generalization by enlarging the smaller (useful) singular values that lie in the \textit{information space},
while preserving eigenvectors. Hence, it enhances learning along these (harder to learn) components.
The following Lemma captures the improvement in the generalization performance, as a result of data augmentation.

\begin{lemma}
\label{cor:generalization}
{Assume gradient descent with learning rate $\eta$ is applied to train a neural network with constant NTK and Lipschitz constant $L$,
on data points %
augmented with %
additive perturbations bounded by $\epsilon_0$ as defined in Sec. \ref{sec:formulation}.} Let $\sigma_{\min}$ be the minimum singular value of Jacobian $\mathcal{J}$ associated with training data $\bm{X}_{train}$. %
With probability $1-\delta$, generalization error of the network trained with gradient descent on augmented data $\X_{aug}$ enjoys the following bound: %
\vspace{-2mm}
\begin{align}
    \sqrt{
    \frac{2}{(\sigma_{\min} + \sqrt{n}L\epsilon_0)^2}} + \mathcal{O}\left(\log{\frac{1}{\delta}}\right).
\end{align}
\end{lemma}
The proof can be found in Appendix {\ref{app:proof-generalization}}.

\section{Effective Subsets for Data Augmentation}\label{sec:coreset-augment}
\begin{algorithm}[t]
	\begin{algorithmic}[1]
		\Require{The dataset $\mathcal{D}=\{(\x_i,y_i)\}_{i=1}^n$, number of iterations $T$.}
		\Ensure{Output model parameters $\W^T$.}
		\For{$t=1, \cdots, T$}
		\State $\X_{aug}^t=\emptyset$.
		\For {$c \in \{1,\cdots,C\}$}
		\State $S_c^t = \emptyset$, $[\bm{G}_{S_c^t}]_{i.}=c_1 \bm{1}~~\forall i$.
		\While {$\|\bm{G}_{S_c^t}\|_F\geq \xi$}\Comment{Extract a coreset from class $c$ by solving Eq. \eqref{eq:fl_min}}
		\State $S_c^t = \{S_c^t \cup {\arg\max}_{s\in V\setminus S_c^t} (\|\bm{G}_{S_c^t}\|_F-\|\bm{G}_{\{S_c^t \cup \{s\}\}}\|_F)\}$
		\EndWhile
		\State $\gamma_j=\sum_{i\in V_c}\!\mathbb{I}[j\!=\!{\arg\min}_{j'\in S}\|\Jc^T(\W^t,\bm{x}_{i})r_{i}\!-\!\Jc^T(\W^t,\bm{x}_{j'})r_{j'}\|]$
		\hspace{-2mm}\Comment{Coreset weights}
		\State $\X_{aug}^t=\{\X_{aug}\cup\{\cup^r_{i=1} T_i^t(\X_{S^t_c})\}\}$\Comment{Augment the coreset}
		\State $\bm{\rho}^t_j=\gamma^t_j/r$ %
		\EndFor
		\State Update the parameters $\W^t$ using weighted gradient descent on $\X_{aug}^t$ or $\{\X_{train} \cup \X_{aug}^t\}$. 
		\EndFor
	\end{algorithmic}
	\caption{\textsc{Coresets for Efficient Data Augmentation}}
	\label{alg:core_aug}
\end{algorithm}

Here, we focus on identifying 
subsets of data that when augmented similarly improve generalization and prevent overfitting.
To do so, our key idea is to find subsets of data points that when augmented, closely capture the alignment of the NTK (or equivalently the Jacobian) corresponding to the full augmented data with the residual vector, 
$\Jc(\W^t,\X^t_{aug})^T\bm{r}_{aug}^t$. %
If such subsets can be found, augmenting only the subsets will change the %
NTK and its alignment with the residual in a similar way as that of full data augmentation, and
will result in similar improved training dynamics.
However, generating the full set of transformations $\X_{aug}^t$ is often very expensive, 
particularly for strong augmentations and large datasets. 
Hence, %
generating the transformations, %
and then extracting the subsets may not provide a considerable 
overall speedup.

In the following, we show that weighted subsets (coresets) $S$ that closely estimate the alignment of the Jacobian associated to the original data with the residual vector $\Jc^T(\W^t,\X_{train})\bm{r}_{train}$ %
can closely estimate the alignment of the Jacobian of the full augmented data and the corresponding residual $\Jc^T(\W^t,\X^t_{aug})\bm{r}_{aug}^t$. Thus, the most effective subsets for augmentation can be directly found from the training data.
Formally, subsets $S_*^t$ weighted by $\bm{\gamma}_S^t$ that capture the alignment of the full Jacobian with the residual by an error of at most $\xi$ can be found by
solving the following optimization problem: \looseness=-1
\begin{align}
\label{eq:coreset_error}
    S^t_*=&\argmin_{S\subseteq V} |S| \quad\quad \text{s.t.} \quad\quad
    \|\Jc^T(\W^t,\X^t)\bm{r}^t-\text{diag}(\bm{\gamma}_{S}^t)\Jc^T({\W}^t,\X_S^t)\bm{r}_S^t\| \leq \xi.%
\vspace{-2mm}    
\end{align}
Solving the above optimization problem is NP-hard. However, as we discuss in the Appendix {\ref{app:sub}}, 
a near optimal subset can be found by minimizing
the Frobenius norm of a matrix $\bm{G}_S$, in which the $i^{th}$ row contains the %
euclidean distance between data point $i$ and its closest element in the subset $S$, in the gradient space. Formally, $[\bm{G}_S]_{i.}={\min}_{j'\in S}\|\Jc^T(\W^t,\bm{x}_{i})r_{i}-\Jc^T(\W^t,\bm{x}_{j'})r_{j'}\|$. When $S=\emptyset$, $[\bm{G}_S]_{i.}=c_1 \bm{1}$, where $c_1$ is a big constant.

Intuitively, such subsets contain the set of medoids of the dataset in the gradient space.
Medoids of a dataset are defined as the most centrally located elements in the dataset \citep{kaufman1987clustering}.
The weight of every element $j\in S$ is the number of data points closest to it in the gradient space, i.e., $\gamma_j=\sum_{i\in V}\mathbb{I}[j={\arg\min}_{j'\in S}\|\Jc^T(\W^t,\bm{x}_{i})r_{i}-\Jc^T(\W^t,\bm{x}_{j'})r_{j'}\|]$.
The set of medoids can be found by solving the following \textit{submodular}%
\footnote{A set function $F:2^V \rightarrow \R^+$ is submodular if $F(S\cup\{e\}) - F(S) \geq F(T\cup\{e\}) - F(T),$ for any $S\subseteq T \subseteq V$ and $e\in V\setminus T$. %
$F$ is \textit{monotone} if $F(e|S)\geq 0$ for any $e\!\in\!V\!\setminus \!S$ and $S\subseteq V$.}  cover problem: 
\begin{equation}\label{eq:fl_min}
S_*^t={\arg\min}_{S\subseteq V}|S|\quad s.t.\quad \|\bm{G}_S\|_F\leq \xi.
\end{equation}
The classical greedy algorithm provides a %
logarithmic approximation for the above submodular maximization problem, i.e., $|S|\leq (1+ln(n))$. It starts with the empty set $S_0=\emptyset$, and at each iteration $\tau$, it selects the training example $s\in V\setminus S_{\tau-1}$ that maximizes the marginal gain, i.e., %
$S_\tau = S_{\tau-1}\cup\{{\arg\max}_{s\in V\setminus S_{\tau-1}} (\|\bm{G}_{S_{\tau-1}}\|_F-\|\bm{G}_{\{S_{\tau-1} \cup \{s\}\}}\|_F)\}$. 
The $\mathcal{O}(nk)$ 
computational complexity of the greedy algorithm %
can be reduced to $\mathcal{O}(n)$ using randomized methods \citep{mirzasoleiman2015lazier} and further improved using lazy evaluation \citep{minoux1978accelerated} and distributed implementations \citep{mirzasoleiman2013distributed}. 
The rows of the matrix $\bm{G}$ can be efficiently upper-bounded using
the gradient of the loss w.r.t. the input to the last layer of the network, which has been shown to capture the variation of the gradient norms closely \citep{katharopoulos2018not}. 
The above upper-bound is only marginally more expensive than calculating the value of the loss. %
Hence the subset can be found efficiently.
Better approximations can be obtained by considering earlier layers in addition to the last two, at the expense of greater computational cost. %

At every iteration $t$ during training, %
we select a coreset from every class $c\in [C]$ separately, and apply
the set of transformations $\{T_i^t\}_{i=1}^{r}$
only to the elements of the coresets, i.e., $X_{aug}^t=\{\cup^r_{i=1} T_i^t(\X_{S^t})\}$.
We divide the weight of every element $j$ in the coreset equally among its transformations, i.e. the final weight 
${\rho}^t_j=\gamma^t_j/r$ if $j \in S^t$. %
We apply the gradient descent updates in Eq. \eqref{eq:gd} to the weighted Jacobian matrix of $\X^t=\X_{aug}^t$ or $\X^t=\{\X_{train}\cup \X_{aug}^t\}$ (viewing $\bm{\rho}^t$ as $\bm{\rho}^t \in \mathbb{R}^{n}$) as follows: 
\begin{align}
    \bm{W}^{t + 1} = \bm{W}^{t} - \eta \left(\text{diag}(\bm{\rho}^t)\Jc(\W^t,\X^t)\right)^T\bm{r}^t.
\end{align} 
The pseudocode is illustrated in Alg. \ref{alg:core_aug}.

The following Lemma upper bounds the difference between the alignment of the Jacobian and residual for augmented coreset vs. full augmented data.
\begin{lemma}
\label{lemma:gradient-coreset-augment}
 Let $S$ be a coreset that captures the alignment of the full data NTK with residual with an error of at most $\xi$ as in Eq. \ref{eq:coreset_error}.
 Augmenting the coreset with perturbations bounded by $\epsilon_0 \!\leq\! \frac{1}{n^{\frac{3}{2}}\sqrt{L}}$%
 captures the alignment of the fully augmented data with the residual by an error of at most 
 \begin{align}
 \vspace{-2mm}
\hspace{-3mm}\| \Jc^T(\W^t,\X_{aug})\bm{r} - \text{diag}(\bm{\rho}^t)\Jc^t(&\W^t,\X_{S^{aug}})\bm{r}_{S}\| 
 \leq \xi + \mathcal{O}\left(\sqrt{L}\right).
 \end{align}
\end{lemma}

\vspace{-3mm}
\subsection{Coreset vs. Max-loss Data Augmentation}\vspace{-1mm}
In the initial phase of training %
the NTK goes through rapid changes. %
This determines the final basin of convergence and network’s final performance \citep{fort2020deep}.
Regularizing deep networks by weight decay or data augmentation mainly affects this initial phase and matters little afterwards \citep{golatkar2019time}.
Crucially, augmenting coresets that closely capture the alignment of the NTK with the residual during this initial phase results in less overfitting and 
improved generalization performance.
On the other hand, augmenting points with maximum loss early in training decreases the alignment between the NTK and the label vector and impedes learning and convergence.
After this initial phase when the network has good prediction performance, the gradients for majority of data points become small. Here, the alignment is mainly captured by the elements with the maximum loss. 
Thus, %
as training proceeds, the intersection between the elements of the coresets and examples with maximum loss increases. We visualize this pattern in Appendix {\ref{app:intersection-maxloss-coreset}}. %
The following Theorem characterizes the training dynamics of training on the full data and the augmented coresets, using our additive perturbation model.

\begin{table*}[t]
 \caption{Training ResNet20 (R20) and WideResnet-28-10 (W2810) on CIFAR10 (C10) using small subsets, and ResNet18 (R18) on Caltech256 (Cal). We compare accuracies of training on and strongly (and weakly) augmenting subsets.
 For CIFAR10, training and augmenting subsets selected by max-loss performed poorly and did not converge. Average number of examples per class in each subset is shown in parentheses. %
 {Appendix \ref{app:results-train-subset-include-weak-aug-only} shows baseline accuracies from only weak augmentations.}
 }
\centering
\footnotesize
\setlength{\tabcolsep}{2.1pt}
\resizebox{\textwidth}{!}
    {
\begin{tabular}{cccccccccccc}
\toprule
  \multicolumn{1}{c}{Model/Data}
  & \multicolumn{4}{c}{C10/R20}
  & \multicolumn{1}{c}{C10/W2810}
  & \multicolumn{6}{c}{Cal/R18} \\
   \cmidrule(lr){2-5} \cmidrule(lr){6-6}  \cmidrule(lr){7-12}
   Subset & 0.1\% (5) & 0.2\% (10) & 0.5\% (25) & 1\% (50) & 1\% (50) & 5\% (3) & 10\% (6) & 20\% (12) & 30\% (18) & 40\% (24) & 50\% (30) \\
  \midrule
  Max-loss &  $<15\%$ &  $<15\%$ &  $<15\%$ &  $<15\%$ & $<15\%$ & $19.2$ & $50.6$ & $71.3$ & $75.6$ & $77.3$ & $78.6$\\
  Random & $33.5$ & $42.7$ & $58.7$ & $74.4$ & $57.7$ & $41.5$ & $61.8$ & $72.5$ & $75.7$ & $77.6$ & $78.5$\\
  Ours & \bm{$37.8$} & \bm{$45.1$} & \bm{$63.9$} & \bm{$74.7$} & \bm{$62.1$} & \bm{$52.7$} & \bm{$65.4$} & \bm{$73.1$} & \bm{$76.3$} & \bm{$77.7$} & \bm{$78.9$} \\
  \bottomrule
\end{tabular}
}
\label{tab:results-train-subset}
\vspace{-4mm}
\end{table*}

\begin{theorem} %
\label{theorem:alpha-pl}
Let $\mathcal{L}_{i}$ be $\beta$-smooth, $\mathcal{L}$ be $\lambda$-smooth and satisfy the $\alpha$-PL condition, that is for $\alpha > 0$, ${\| \nabla \mathcal{L}(\bm{W}) \|}^2 \geq \alpha \mathcal{L}(\bm{W})$ for all weights $\bm{W}$. Let $f$ be Lipschitz in $\bm{X}$ with constant $L'$, and $\bar{L} =  \max\{L,L'\}$. Let $G_0$ be the gradient at initializaion, $\sigma_{\max}$ the maximum singular value of the coreset Jacobian at initialization. Choosing $\epsilon_0 \leq \frac{1}{\sigma_{\max}\sqrt{\bar{L}n}}$ and running SGD on full data with augmented coreset using constant step size $\eta = \frac{\alpha}{\lambda\beta}$, result in the following bound:
\begin{align}
\vspace{-3mm}
    \mathbb{E} [\| \nabla& \mathcal{L}^{f + c_{\aug}}(\bm{W}^t) \|] \leq  \frac{1}{\sqrt{\alpha}} \left( 1 - \frac{\alpha \eta}{2} \right)^{\frac{t}{2}} \left(2G_0 + \xi + \mathcal{O}\left(\frac{\sqrt{\bar{L}}}{\sigma_{\max}}\right) \right).\nonumber
\end{align}
\vspace{-5mm}
\end{theorem}
The proof can be found in Appendix {\ref{app:proof-alpha-pl}}.

Theorem \ref{theorem:alpha-pl} shows that training on full data and augmented coresets converges to a close neighborhood of the optimal solution, with the same rate as that of training on the fully augmented %
data. The size of the neighborhood depends on the error of the coreset $\xi$ in Eq. \eqref{eq:coreset_error}, and the error in capturing the alignment of the full augmented data with the residual derived in Lemma \ref{lemma:gradient-coreset-augment}. The first term decrease as the size of the coreset grows, and the second term depends on the network structure. 

We also analyze convergence of training only on the augmented coresets, and  
augmentations modelled as arbitrary
linear transformations using a linear model %
\citep{wu2020generalization} in Appendix {\ref{app:aug-linear-transform}}.

\vspace{-2mm}
\section{Experiments}
\vspace{-1mm}
\label{sec:experiments}

\textbf{Setup and baselines.} We extensively evaluate the performance of our approach in three different settings. 
Firstly, %
we consider training only on coresets and their augmentations. Secondly, we investigate the effect of adding augmented coresets to the full training data. Finally, we consider adding augmented coresets to random subsets.
We compare our coresets with max-loss and random subsets as baselines. 
For all methods, we select a new augmentation subset every $R$ epochs.
We note that the original max-loss method \citep{kuchnik2018efficient} selects points using a fully trained model, hence it can only select one subset throughout training. To maximize fairness, we modify our max-loss baseline to select a new subset at every subset selection step. 
For all experiments, standard weak augmentations (random crop and horizontal flips) are always performed on both the original and strongly augmented data.\looseness=-1

 \vspace{-2mm}
 \subsection{Training on Coresets and their Augmentations} \vspace{-2mm}
First, we evaluate the effectiveness of our approach for training on the coresets and their augmentations. Our main goal here is to 
compare the performance of training on and augmenting coresets vs. random and max-loss subsets.
\tabref{tab:results-train-subset} shows the test accuracy for training ResNet20 and Wide-ResNet on CIFAR10 when we only train on small augmented coresets of size $0.1\%$ to $1\%$ selected at every epoch ($R=1$), and training ResNet18 on Caltech256 using coresets of size $5\%$ to $50\%$ with $R=5$. We see that the augmented coresets outperform augmented random subsets by a large margin, particularly when the size of the subset is small.
On Caltech256/ResNet18, training on and augmenting 10\% coresets %
yields 65.4\% accuracy, improving over random by 5.8\%, and over only weak augmentation by 17.4\%. %
This clearly shows the effectiveness of augmenting the coresets. %
Note that for CIFAR10 experiments, training on the augmented max-loss points did not even converge in absence of full data. 

\textbf{Generalization across augmentation techniques.}  
We note that our coresets are not dependent on the type of data augmentation. To confirm this, we show the superior generalization performance of our method in \tabref{tab:augmentations} for training ResNet18 with $R=5$ on coresets vs. random subsets of Caltech256, augmented
with CutOut \cite{devries2017improved}, 
AugMix \cite{hendrycks2019augmix}, 
and noise perturbations (color jitter, gaussian blur). For example, on $30\%$ subsets, we obtain 28.2\%, 29.3\%, 20.2\% relative improvement over augmenting random subsets when using CutOut, AugMix, and {noise} perturbation augmentations, respectively. %

\begin{table}[t]
    \parbox{.45\linewidth}{
    \caption{Caltech256/ResNet18 with same settings as \tabref{tab:results-train-subset} with default weak augmentations but varying strong augmentations.}
    \label{tab:augmentations}
    \vspace{-2mm}
    \centering
    \footnotesize
    \setlength{\tabcolsep}{2pt}
    \resizebox{0.45\textwidth}{!}
    {%
    \begin{tabular}{ccccccc}
    \toprule
    \multirow{2}{*}{Augmentation} 
    & \multicolumn{3}{c}{Random} & \multicolumn{3}{c}{Ours} \\
        \cmidrule(lr){2-4}
        \cmidrule(lr){5-7}
        & 30\% & 40\% & 50\% & 30\% & 40\% & 50\% \\
    \midrule
    CutOut & 43.32 & 62.84 & 76.21 & \textbf{55.53} & \textbf{66.10} & \textbf{76.91} \\
    AugMix & 40.77 & 61.81 & 72.17 & \textbf{52.72} & \textbf{64.91} & \textbf{73.01} \\
    Perturb & 48.51 & 66.20 & 75.34 & \textbf{58.29} & \textbf{67.47} & \textbf{76.50} \\
    \bottomrule
    \end{tabular}
    }
    }\hfill
    \parbox{.5\linewidth}{
    \centering
    \footnotesize
    \caption{Training on full data and strongly (and weakly) augmenting random subsets, max-loss subsets and coresets on TinyImageNet/ResNet50, $R=15$.}
    \label{tab:tinyimagenet}
    \setlength{\tabcolsep}{2pt}
     \resizebox{0.5\textwidth}{!}{
    \begin{tabular}{ccccccccc}
    \toprule
    \multicolumn{3}{c}{Random} & \multicolumn{3}{c}{Max-loss} & \multicolumn{3}{c}{Ours} \\
            \cmidrule(lr){1-3}
        \cmidrule(lr){4-6}
          \cmidrule(lr){7-9}
        $20\%$ & $30\%$ & $50\%$ & $20\%$ & $30\%$ & $50\%$ & $20\%$ & $30\%$ & $50\%$ \\
    \midrule
    50.97 & 52.00 & 54.92 & 51.30 & 52.34  & 53.37 & \textbf{51.99}   & \textbf{54.30} & \textbf{55.16} \\
    \bottomrule
    \end{tabular}
    }
    }
    \vspace{-4mm}
\end{table}

\vspace{-2mm}
\subsection{Training on Full Data and Augmented Coresets}\label{exp:full} \vspace{-2mm}

\begin{table*}[t]
 \caption{Accuracy improvement by augmenting subsets found by our method vs. max-loss and random, over improvement of full (weak and strong) data augmentation (F.A.) compared to weak augmentation only (W.A.). The table shows the results for training on CIFAR10\! (C10)/ResNet20 \!(R20), SVHN/ResNet32\! (R32), and CIFAR10-Imbalanced\! (C10-IB)/ResNet32, with  %
 $R=20$. %
 } %
    \centering
    \footnotesize
    \resizebox{\textwidth}{!}
    {%
    \begin{tabular}{*{12}c}
  \toprule
  \multicolumn{1}{c}{Dataset}
  & \multicolumn{1}{c}{W.A.}
  & \multicolumn{1}{c}{F.A.}
  & \multicolumn{3}{c}{Random}
  & \multicolumn{3}{c}{Max-loss}
  & \multicolumn{3}{c}{Ours} \\
    \cmidrule(lr){2-2}
    \cmidrule(lr){3-3}
     \cmidrule(lr){4-6}
     \cmidrule(lr){7-9}
     \cmidrule(lr){10-12}
  & \multicolumn{1}{c}{Acc}
  & \multicolumn{1}{c}{Acc}
  & \multicolumn{1}{c}{$5\%$}
  & \multicolumn{1}{c}{$10\%$}
  & \multicolumn{1}{c}{$30\%$}
  & \multicolumn{1}{c}{$5\%$}
  & \multicolumn{1}{c}{$10\%$}
  & \multicolumn{1}{c}{$30\%$}
  & \multicolumn{1}{c}{$5\%$}
  & \multicolumn{1}{c}{$10\%$}
  & \multicolumn{1}{c}{$30\%$} \\
     \midrule
     C10/R20
     & $89.46$
     & $93.50$
     & $21.8\%$
     & $39.9\%$
     & $65.6\%$
     & $32.9\%$
     & $47.8\%$
     & $73.5\%$
     & $\bm{34.9\%}$
     & $\bm{51.5\%}$
     & $\bm{75.0\%}$
     \\   
     C10-IB/R32
     & $87.08$
     & $92.48$
     & $25.9\%$
     & $45.2\%$
     & $74.6\%$
     & $31.3\%$
     & $39.6\%$
     & $74.6\%$
     & $\bm{37.4\%}$
     & $\bm{49.4\%}$
     & $\bm{74.8\%}$
     \\   
     SVHN/R32
     & $95.68$
     & $97.07$
     & $5.8\%$
     & $36.7\%$
     & $64.1\%$
     & $\bm{35.3\%}$
     & $\bm{49.7\%}$
     & $76.4\%$
     & $31.7\%$
     & $48.3\%$
     & $\bm{80.0\%}$ \\   
     \bottomrule
 \end{tabular}
 }
\label{tab:results-main}
\vspace{-5mm}
\end{table*}
Next, we study the effectiveness of our method for training on full data and augmented coresets.
\tabref{tab:results-main}
demonstrates the percentage of
accuracy improvement resulted by augmenting subsets of size 5\%, 10\%, and 30\% selected from our method vs. max-loss and random subsets, over that of full data augmentation. 
We observe that augmenting coresets effectively improves generalization, and outperforms augmenting random and max-loss subsets across different models and datasets. For example, on 30\% subsets, we obtain 13.1\% and 2.3\% relative improvement over random and max-loss on average. %
We also report results on TinyImageNet/ResNet50 ($R\!=\!15$) in \tabref{tab:tinyimagenet}, where we show that augmenting coresets outperforms max-loss and random baselines, e.g. by achieving 3.7\% and 4.4\% relative improvement over 30\% max-loss and random subsets, respectively.

\textbf{Training speedup.} In \figref{fig:speedup1}, we measure the improvement in training time in the case of training on full data and augmenting subsets of various sizes. While our method yields similar or slightly lower speedup to the max-loss and random baselines, our resulting accuracy outperforms the baselines on average. For example, for SVHN/Resnet32 using $30\%$ coresets, we sacrifice $11\%$ of the {relative} %
speedup to obtain an additional $24.8\%$ of the {relative} %
gain in accuracy from full data augmentation, compared to random baseline.
Notably, we get 3.43x speedup for training on full data and augmenting 30\% coresets, while obtaining 75\% of the improvement of full data augmentation.
We provide wall-clock times for finding coresets from Caltech256 and TinyImageNet in Appendix {\ref{app:speed-up-measurements}}.
\begin{figure}[h]
\centering
\vspace{-2mm}
\subfigure[CIFAR10/ResNet20 \label{fig:cifar-speedup}]{
\includegraphics[width=0.2\textwidth,trim=10mm 0 12mm 10mm]{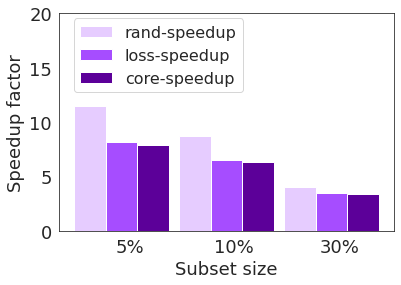}
}\hspace{3mm}
\subfigure[CIFAR10/ResNet20 \label{fig:cifar-acc}]{
\includegraphics[width=0.2\textwidth,trim=10mm 0 12mm 10mm]{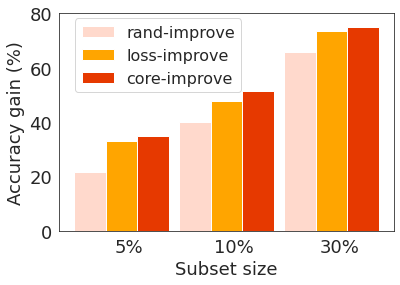}
}\hspace{3mm}
\subfigure[SVHN/ResNet32 \label{fig:svhn-speedup}]{
\includegraphics[width=0.2\textwidth,trim=10mm 0 12mm 10mm]{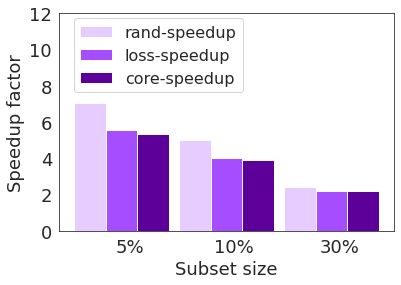}
}\hspace{3mm}
\subfigure[SVHN/ResNet32 \label{fig:svhn-acc}]{
\includegraphics[width=0.2\textwidth,trim=10mm 0 12mm 10mm]{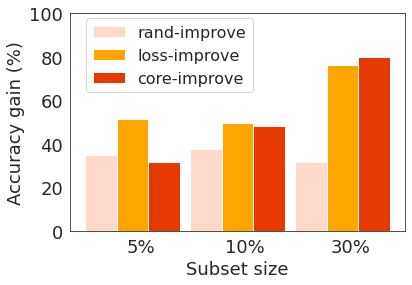}
}
\vspace{-6mm}
\caption{Accuracy improvement and speedups by augmenting subsets found by our method vs. max-loss and random on (a), (b) ResNet20/CIFAR10 and (c), (d) ResNet32/SVHN. %
 }\label{fig:speedup1}\vspace{-5mm}
\end{figure}

\textbf{Augmenting noisy labeled data.}
\begin{table}[t]
 \caption{Training ResNet20 on CIFAR10 with $50\%$ label noise, $R=20$.
 Accuracy without strong augmentation is $70.72 \pm 0.20$ and the accuracy of full (weak and strong) data augmentation is $75.87 \pm 0.77$.
 Note that augmenting $50\%$ subsets outperforms augmenting the full data (marked $\bm{^{**}}$).}
  \vspace{-2mm}
    \centering
    \footnotesize
    \resizebox{0.5\textwidth}{!}
    {
      \begin{tabular}{*{4}c}
  \toprule
  \multicolumn{1}{c}{Subset}
  & \multicolumn{1}{c}{Random}
  & \multicolumn{1}{c}{Max-loss} 
  & \multicolumn{1}{c}{Ours} 
  \\
     \midrule
     10\% 
     & $72.32 \pm 0.14$
     & $71.83 \pm 0.13$
     & \bm{$73.02 \pm 1.06$}
     \\   
     30\%
     & $74.46 \pm 0.27$
     & $72.45 \pm 0.48$
     & \bm{$74.67 \pm 0.15$}
     \\   
     50\%
     & $75.36 \pm 0.05$
     & $73.23 \pm 0.72$
     & \bm{$76.20 \pm 0.75^{**}$} 
     \\
     \bottomrule
 \end{tabular}
 }
 \vspace{-3mm}
\label{tab:results-noisy}
\vspace{-1mm}
\end{table}
Next, we evaluate the robustness of our coresets to label noise.
\tabref{tab:results-noisy} shows the result of augmenting coresets vs. max-loss and random subsets of different sizes selected from CIFAR10 with 50\% label noise on ResNet20. Notably, our method not only outperforms max-loss and random baselines, %
but also achieves superior performance over full data augmentation. 

\vspace{-3mm}
\subsection{Training on Random Data and Augmented Coresets}\vspace{-2mm}
Finally, we evaluate the performance of our method for training on random subsets and augmenting coresets, applicable when data is larger than the training budget.
We report results on TinyImageNet and ImageNet on ResNet50 (90 epochs, $R\!=\!15$).  \tabref{tab:tinyimagenet2} shows the results of training on random subsets, and augmenting random subsets and coresets of the same size. We see that our results hold for large-scale datasets, where we obtain $7.9\%$, $4.9\%$, and $5.3\%$ relative improvement over random baseline with $10\%$, $20\%$, $30\%$ subset sizes respectively on TinyImageNet, and $7.6\%$, $2.3\%$, and $1.3\%$ relative improvement over random baseline with $10\%$, $30\%$, and $50\%$ subset sizes on ImageNet. %
{Notably, compared to AutoAugment, despite using only 30\% subsets, we achieve 71.99\% test accuracy, which is 92.8\% of the original reported accuracy, while boasting 5x speedup in training.}

\begin{table}[H]
    \vspace{-1mm}
    \centering
    \footnotesize
    \caption{Training on random subsets and strongly (and weakly) augmenting random and max loss subsets vs coresets for TinyImageNet (left) and ImageNet (right) with ResNet50.}
    \vspace{-2mm}
    \label{tab:tinyimagenet2}
    \setlength{\tabcolsep}{2pt}
    \resizebox{0.9\textwidth}{!}
    {
    \begin{tabular}{ccccccccc}
    \toprule
        \multicolumn{3}{c}{Random} & \multicolumn{3}{c}{Max-loss} & \multicolumn{3}{c}{Ours} \\
        \cmidrule(lr){1-3} \cmidrule(lr){4-6} \cmidrule(lr){7-9}
        10\% & 20\% & 30\% & 10\% & 20\% & 30\% & 10\% & 20\% & 30\%  \\
    \midrule
   	28.64 & 38.97 & 44.10 & 27.64 & \textbf{41.40} & 45.75 & \textbf{30.90} & 40.88 & \textbf{46.42} \\
    \bottomrule
    \end{tabular} \ 
    \begin{tabular}{ccccccccc}
    \toprule
        \multicolumn{3}{c}{Random} & \multicolumn{3}{c}{Maxloss} &  \multicolumn{3}{c}{Ours} \\
        \cmidrule(lr){1-3} \cmidrule(lr){4-6} \cmidrule(lr){7-9}
        10\% & 30\% & 50\% & 10\% & 30\% & 50\%  & 10\% & 30\% & 50\%  \\
    \midrule
   	63.67 & 70.39 &  72.35 & 65.43 & 71.55 & 72.77 & \textbf{68.53} & \textbf{71.99} & \textbf{73.28} \\
    \bottomrule
    \end{tabular}
    }
    \vspace{-5mm}
\end{table}

\vspace{-2mm}
\section{Conclusion}
\vspace{-3mm}
We showed that data augmentation improves training and generalization by relatively enlarging and perturbing the smaller singular values of the neural network Jacobian while preserving its prominent directions.
Then, we proposed a framework to iteratively extract small coresets of training data that when augmented, closely capture the alignment of the fully augmented Jacobian with the label/residual vector.
We showed the effectiveness of augmenting coresets in providing a superior generalization performance when added to the full data or random subsets, in presence of noisy labels, or as a standalone subset.
Under local smoothness of images, our additive perturbation can be applied to model many bounded transformations such as small rotations, crops, shearing, and pixel-wise transformations like sharpening, blurring, color distortions, structured adversarial perturbation \cite{luo2020data}. However, 
the additive perturbation model is indeed limited when applied to augmentations that cannot be reduced to perturbations, such as horizontal/vertical flips and large translations. 
Further theoretical analysis of complex data augmentations is indeed an interesting direction for future work.

\vspace{-3mm}
\section{Acknowledgements}
\vspace{-3mm}
This research was supported in part by %
the National Science Foundation CAREER Award 2146492, and the UCLA-Amazon Science Hub for Humanity and AI.

\bibliography{arxiv}
\bibliographystyle{plain}

\newpage
\appendix

\section*{\begin{center}\Large{Supplementary Material:\\ Data-Efficient Augmentation for Training Neural Networks}\end{center}}

\section{Proof of Main Results}

\subsection{Proof for Lemma \ref{lemma:singular-value-perturbation-bounds}}
\label{appendix:proof-singular-value-perturbation-bounds}
\begin{proof}
Let $\delta_i := \tilde{\sigma}_i - \sigma_i$, where $\mathbb{P}(\delta_i < 0) = p_i$.
Assuming uniform probability between $- \| \bm{E} \|$ to $0$, and between $0$ to $\|\bm{E}\|$, we have pdf $\rho_i(x)$ for $\delta_i$:
\begin{align}
\rho_i(x) = \begin{cases}
\frac{p_i}{\|\bm{E}\|} ,& \text{if } -\|\bm{E}\| \leq x < 0\\
\frac{1-p_i}{\| \bm{E} \|},              & 0 \leq x \leq \|\bm{E}\| \\
    0, & \text{otherwise}
\end{cases}
\end{align}
Taking expectation,
\begin{align}
\mathbb{E}(\tilde{\sigma}_i - \sigma_i) = \mathbb{E} (\delta_i) &= \int\limits_{-\infty}^{\infty}  x \rho_i(x) dx \\
&= \int\limits_{-\|\bm{E}\|}^{0} x \frac{p_i}{\| \bm{E} \|} dx + 
\int_{0}^{\|\bm{E}\|} x \frac{1-p_i}{\| \bm{E} \|} dx \\
&= -\frac{\|\bm{E}\| p_i}{2} + \frac{(1-p_i)\|\bm{E}\|}{2} \\
&= \frac{(1-2p_i)\|E\|}{2}
\end{align}
We also have
\begin{align}
     \mathbb{E} (\delta_i^2) &= \int\limits_{-\infty}^{\infty}  x^2 \rho_i(x) dx \\
&= \int\limits_{-\|\bm{E}\|}^{0} x^2 \frac{p_i}{\| \bm{E} \|} dx + 
\int_{0}^{\|\bm{E}\|} x^2 \frac{1-p_i}{\| \bm{E} \|} dx \\
&= \frac{\|\bm{E}\|^2 p_i}{3} + \frac{(1-p_i)\|\bm{E}\|^2}{3} \\
&= \frac{\|\bm{E}\|^2}{3}
\end{align}
Thus, we have
\begin{align}
\mathbb{E}(\tilde{\lambda}_i) &= \mathbb{E} ( \tilde{\sigma}_i^2 ) \\
&= \mathbb{E} ( (\sigma_i + \delta_i)^2 ) \\
&= \mathbb{E} ( \sigma_i^2 + 2 \sigma_i \delta_i + \delta_i^2) \\
&= \sigma_i^2 + 2 \sigma_i\mathbb{E}[\delta_i] + \mathbb{E}{[\delta_i^2]} \\ 
&= \sigma_i^2 + 2\sigma_i\frac{(1-2p_i)\|\bm{E}\|}{2} +  \frac{\|\bm{E}\|^2}{3}\\
&= \sigma_i^2 + \sigma_i(1-2p_i)\|\bm{E}\| +  \frac{\|\bm{E}\|^2}{3}
\end{align}

\end{proof}

\subsection{Proof of Corollary \ref{cor:generalization}}
\label{app:proof-generalization}
Under the assumptions of Theorem 5.1 of \cite{arora2019fine}, i.e.  where the minimum eigenvalue of the NTK is $\lambda_{\min}(\Jc\Jc^T) \geq \lambda_0$ for a constant $\lambda_0>0$, 
and training data $\bm{X}$ of size $n$ sampled i.i.d. from distribution ${D}$ and 1-Lipschitz loss $\mathcal{L}$, we have that with probability $\delta/3$, training the over-parameterized neural network with gradient descent for $t \geq \Omega\left(\frac{1}{n\lambda_0}\log\frac{n}{\delta}\right)$ iterations %
results in the following population loss $\mathcal{L}_D$ (generalization error) %
\begin{align}
    \mathcal{L}_D(\bm{W}^t, \bm{X}) \leq \sqrt{
    \frac{2\bm{y}^T(\mathcal{J}\mathcal{J}^T)^{-1}\bm{y}}{n}} + \mathcal{O}\left(\frac{\log{\frac{n}{\lambda_0\delta}}}{n}\right),
\end{align}
with high probability of at least $1-\delta$ over random initialization and training samples.

Hence, using  $\lambda_{\min}, \sigma_{\min}$ to denote minimum eigen and singular value respectively of the NTK corresponding to full data, we get
\begin{align}
    \mathcal{L}_{D_{train}}(\bm{W}^t, \bm{X}_{train})
    &\leq
    \sqrt{
    \frac{2 \frac{1}{\lambda_{\min}} \|y\|^2}{n}} + \mathcal{O}\left(\log{\frac{1}{\delta}}\right) \\
    &\leq \sqrt{
    \frac{2}{\sigma_{\min}^2}} + \mathcal{O}\left(\log{\frac{1}{\delta}}\right).
\end{align}

For augmented dataset $\bm{X}_{aug}$, we have $\tilde{\sigma_i} \leq \sigma_i + \sqrt{n}L\epsilon_0$, hence the improvement in the generalization error is at most
\begin{align}
    \mathcal{L}_{D_{aug}}(\bm{W}^t, \bm{X}_{\aug}) &\leq \sqrt{
    \frac{2}{(\sigma_{\min} + \sqrt{n}L\epsilon_0)^2}} + \mathcal{O}\left(\log{\frac{1}{\delta}}\right).
\end{align}
Combining these two results, we obtain Corollary \ref{cor:generalization}.

\subsection{Proof of Lemma \ref{lemma:gradient-coreset-augment}}

\begin{proof}
\begin{align}
\| \Jc^T(\W^t,\X_{aug})\bm{r} - &\text{diag}(\bm{\rho}^t)\Jc^t(\W^t,\X_{S^{aug}})\bm{r}_{S}\| \\
 &= \| (\Jc^T(\W^t,\X) + \bm{E}) \bm{r} - (\text{diag}(\bm{\rho}^t)\Jc^t(\W^t,\X_{S}) + \bm{E}_{S})\bm{r}_{S}\| \\
 &\leq \| (\Jc^T(\W^t,\X)\bm{r} - \text{diag}(\bm{\rho}^t)\Jc^t(\W^t,\X_{S})\bm{r}_{S}) + \bm{E} \bm{r} - \bm{E}_{S}\bm{r}_{S} \| \\
 &\leq \| (\Jc^T(\W^t,\X)\bm{r} - \text{diag}(\bm{\rho}^t)\Jc^t(\W^t,\X_{S})\bm{r}_{S}) \| + \| \bm{E} \bm{r} \| + \| \bm{E}_{S}\bm{r}_{S} \|
 \end{align}
 Applying definition of coresets, we obtain
 \begin{align}
     \| (\Jc^T(\W^t,\X)\bm{r} - &\text{diag}(\bm{\rho}^t)\Jc^t(\W^t,\X_{S})\bm{r}_{S}) \| + \| \bm{E} \bm{r} \| + \| \bm{E}_{S}\bm{r}_{S} \| \\
     &\leq \xi + \| \bm{E} \bm{r} \| + \| \bm{E}_{S}\bm{r}_{S} \| \\
     &\leq \xi + 2 n^{\frac{3}{2}}L\epsilon_0 \\
     &\leq \xi + 2 \sqrt{L}
 \end{align}
 
\end{proof}

\subsection{Proof of Theorem \ref{theorem:alpha-pl}}
\label{app:proof-alpha-pl}
\begin{proof}
In this proof, as shorthand notation, we use$\bm{X}_f$ and $\bm{X}_{train}$ interchangeably. We further use $\bm{X}_c$ to represent the coreset selected from the full data, and $\bm{X}_{c_{aug}}$ to represent the augmented coreset.

By Theorem 1 of \cite{bassily2018exponential}, under the $\alpha$-PL assumption for $\mathcal{L}$ and interpolation assumption (i.e. for every sequence $\bm{W}^1, \bm{W}^2, \ldots$ such that $\lim_{t\rightarrow\infty} \mathcal{L}(\bm{W}^t, \bm{X}) = 0$, we have that the loss for each data point $\lim_{t\rightarrow\infty} \mathcal{L}(\bm{W}^t, \bm{x}_i) = 0$), the convergence of SGD with constant step size %
is given by
\begin{align}
    \mathbb{E} [\| \nabla \Lc(\bm{W}^t,  \bm{X}_{f + c_{\aug}}) \|^2] &\leq \left( 1 - \frac{\alpha \eta}{2} \right)^{t} \mathcal{L}(\bm{W}^0, \bm{X}_{f+c_{\aug}}) \\
    &\leq \frac{1}{\alpha} \left( 1 - \frac{\alpha \eta}{2} \right)^{t} \| \nabla \Lc(\bm{W}^0,  \bm{X}_{f + c_{\aug}}) \|^2
\end{align}
Using Jensen's inequality, we have
\begin{align}
    &\mathbb{E} [\| \nabla \Lc(\bm{W}^0,  \bm{X}_{f + c_{\aug}}) \|] \\
    &\leq \sqrt{\mathbb{E} [\| \nabla \Lc(\bm{W}^t,  \bm{X}_{f + c_{\aug}}) \|^2] } \\
    &\leq \frac{1}{\sqrt{\alpha}} \left( 1 - \frac{\alpha \eta}{2} \right)^{\frac{t}{2}} \| \nabla \Lc(\bm{W}^0,  \bm{X}_{f + c_{\aug}}) \| \\
    &\leq  \frac{1}{\sqrt{\alpha}} \left( 1 - \frac{\alpha \eta}{2} \right)^{\frac{t}{2}} \left( \| \nabla \Lc(\bm{W}^0,  \bm{X}_{f}) \|  + \| \nabla \Lc(\bm{W}^0,  \bm{X}_{c_{\aug}}) \| \right ) \\
    &\leq \frac{1}{\sqrt{\alpha}} \left( 1 - \frac{\alpha \eta}{2} \right)^{\frac{t}{2}} \left( G_0 + \| (\Jc(\bm{W}^0, \bm{X}_c) + \bm{E}) (\bm{y} - f(\bm{W}^0, \bm{X}_c + \bm{\epsilon}) ) \| \right ) \\
    &\leq \frac{1}{\sqrt{\alpha}} \left( 1 - \frac{\alpha \eta}{2} \right)^{\frac{t}{2}} \\
    &\quad\left( G_0 + \| (\Jc(\bm{W}^0, \bm{X}_c) + \bm{E})^T (\bm{y} - f(\bm{W}^0, \bm{X}_c) - \nabla_x f(\bm{W}^0, \bm{X}_c)^T\bm{\epsilon} - \mathcal{O}(\bm{\epsilon}^T\bm{\epsilon}) \| \right ) \\
    & = \frac{1}{\sqrt{\alpha}} \left( 1 - \frac{\alpha \eta}{2} \right)^{\frac{t}{2}} ( G_0 + \| \nabla L(\bm{W}^0, \bm{X}_c) - (\Jc(\bm{W}^0, \bm{X}_c)^T(\nabla_x f(\bm{W}^0, \bm{X}_c)^T\bm{\epsilon} + \mathcal{O}(\bm{\epsilon}^T\bm{\epsilon})) + \\ 
    &\quad\quad \bm{E} (\bm{y}  - f(\bm{W}^0, \bm{X}_c +  \bm{\epsilon})) \| )\\
    &\leq \frac{1}{\sqrt{\alpha}} \left( 1 - \frac{\alpha \eta}{2} \right)^{\frac{t}{2}} (G_0 + \| \nabla L(\bm{W}^0, \bm{X}_c) - (\Jc(\bm{W}^0, \bm{X}_c)^T(\nabla_x f(\bm{W}^0, \bm{X}_c)^T\bm{\epsilon} + \mathcal{O}(\bm{\epsilon}^T\bm{\epsilon})) \| + \\ 
    &\quad\quad \sqrt{2}\|\bm{E}\| ) \\
    & \leq  \frac{1}{\sqrt{\alpha}} \left( 1 - \frac{\alpha \eta}{2} \right)^{\frac{t}{2}} (G_0 + \| \nabla L(\bm{W}^0, \bm{X}_c) \| + \sigma_{\max} \bar{L} \sqrt{n} \epsilon_0 + \sigma_{\max} \mathcal{O}(n\epsilon_0^2)) + \sqrt{2n}\bar{L}\epsilon_0) \\
    &= \frac{1}{\sqrt{\alpha}} \left( 1 - \frac{\alpha \eta}{2} \right)^{\frac{t}{2}} (G_0 + \| \nabla L(\bm{W}^0, \bm{X}_f) \| + \xi + \sigma_{\max} \bar{L} \sqrt{n} \epsilon_0 + \sigma_{\max} \mathcal{O}(n\epsilon_0^2)) + \sqrt{2n}\bar{L}\epsilon_0) \\
    &\leq \frac{1}{\sqrt{\alpha}} \left( 1 - \frac{\alpha \eta}{2} \right)^{\frac{t}{2}} (2G_0 + \xi + \sigma_{\max} \bar{L} \sqrt{n} \epsilon_0 + \sigma_{\max} \mathcal{O}(n\epsilon_0^2)) + \sqrt{2n}\bar{L}\epsilon_0) 
\end{align}
\end{proof}

\subsection{Finding Subsets}\label{app:sub}
Let $S$ be a subset of training data points. Furthermore, assume that there is a mapping $\pi_{w,S}: V \rightarrow S$ that for every $\W$ assigns every data point $i\in V$ to its closest element $j\in S$, i.e. $j=\pi_{w,S}(i)=\argmax_{j'\in S}s_{ij'}(\W)$, where $s_{ij}(\W)=C-\|\Jc^T(\W^t,\bm{x}_{i})r_{i}-\Jc^T(\W^t,\bm{x}_{j})r_{j}\|$ is the similarity between gradients of $i$ and $j$, and $C\geq \max_{ij} s_{ij}(\W)$ is a constant. 
Consider a matrix $\bm{G}_{\pi_{w,S}}\in\mathbb{R}^{n\times m}$, in which every row $i$ contains gradient of $\pi_w(i)$, i.e., $[\bm{G}_{\pi_{w,S}}]_{i.}=\Jc^T(\W^t,\bm{x}_{\pi_{w,S}(i)})r_{\pi_{w,S}(i)}$. %
The Frobenius norm of the matrix $\bm{G}_{\pi_w}$ provides an upper-bound on the error of the weighted subset $S$ in capturing the
alignment of the residuals of the full training data with the Jacobian matrix. Formally,
\begin{equation}\label{eq:grad}
\|\Jc^T(\W^t,\X_{train})\bm{r}_{train}^t-\bm{\gamma}_{S^t}\Jc^T({\W}^t,[\X_{{train}}]_{.{S^t}})\bm{r}_{S^t}\|
\leq%
\|\bm{G}_{\pi_{w,S}}\|_F,
\end{equation}
where the weight vector $\bm{\gamma}_{S^t}\in\mathbb{R}^{|S|}$ %
contains the number of elements that are mapped to every element $j\in S$ by mapping $\pi_{w,S}$, i.e. $\gamma_j=\sum_{i\in V}\mathbb{1}[\pi_{w,S}(i)=j]$.
Hence, the set of training points that closely estimate the projection of the residuals of the full training data on the Jacobian spectrum can be obtained %
by finding a subset $S$ that minimizes the Frobenius norm of matrix $\bm{G}_{\pi_{w,S}}$.

\section{Additional Theoretical Results}
\subsection{Convergence analysis for training on augmented full data}
\label{app:full-data-augment-dynamics}
\begin{theorem}
\label{theorem:augmentation2}
Gradient descent with learning rate $\eta$ applied to a neural network with constant NTK and Lipschitz constant $L$, %
and data points $\mathcal{D}_{aug}$ augmented with $r$ additive perturbations bounded by $\epsilon_0$ 
results in the following training dynamics:
\begin{align}
\begin{split}
\vspace{-3mm}
     &\E[\| \bm{y}-f(\X_{aug},\W^t) \|_2] \leq \\ 
     & \sqrt{ \sum_{i=1}^n  \left(1-\eta \left( \sigma_i^2 + \sigma_i(1-2p_i)\|E\| +  \frac{\|E\|^2}{3} \right) \right)^{2t} ((\bm{u}_i \bm{y})^2 + 2n\sqrt{2} \|E\| / \gamma_0) }
\end{split}
\end{align} %
where $\bm{E}$ with $\|\bm{E}\|\leq\sqrt{n}L\epsilon_0$ is the perturbation to the Jacobian, and $p_i := \mathbb{P}(\tilde{\sigma_i} - \sigma_i < 0)$ is the probability that $\sigma_i$ decreases as a result of data augmentation. 
\end{theorem}%

\subsection{Proof of Theorem \ref{theorem:augmentation2}}
Using Jensen's inequality, we have 

\begin{align}
    \mathbb{E} & \left[ \| \bm{y}-f(\bm{X}_{aug},\bm{W}^t) \|_2 \right] \\
    &= \mathbb{E}\left[ {\sqrt{\sum_{i=1}^n (1-\eta\tilde{\lambda}_i)^{2t}( \bm{\tilde{u}}_i^T \bm{y} )^2} \pm\epsilon} \right] \\ 
    &\leq \sqrt{ \mathbb{E}\left[ \sum_{i=1}^n (1-\eta\tilde{\lambda}_i)^{2t}( \bm{\tilde{u}}_i^T \bm{y} )^2 \right]}  \\ 
    &\leq \sqrt{  \sum_{i=1}^n \mathbb{E}\left[ (1-\eta\tilde{\lambda}_i)^{2t} ((\bm{u}_i \bm{y})^2 + 2n\sqrt{2} \|E\| / \gamma_0) \right]}  \\
    &\leq \sqrt{ \sum_{i=1}^n  (1-\eta \mathbb{E}\left[ \tilde{\lambda}_i\right])^{2t} ((\bm{u}_i \bm{y})^2 + 2n\sqrt{2} \|E\| / \gamma_0)} \\
    &= \sqrt{ \sum_{i=1}^n  \left(1-\eta \left( \sigma_i^2 + \sigma_i(1-2p_i)\|E\| +  \frac{\|E\|^2}{3} \right) \right)^{2t} ((\bm{u}_i \bm{y})^2 + 2n\sqrt{2} \|E\| / \gamma_0) } 
\end{align}

\subsection{Convergence analysis for training on the coreset and its augmentation}
\label{subsection:convergence-coreset-augment-only}
\begin{theorem}
Let $\mathcal{L}_{i}$ be $\beta$-smooth, $\mathcal{L}$ be $\lambda$-smooth and satisfy the $\alpha$-PL condition, that is for $\alpha > 0$, ${\| \nabla \mathcal{L}(\bm{W}, \bm{X}) \|}^2 \geq \alpha \mathcal{L}(\bm{W}, \bm{X})$ for all weights $\bm{W}$. Let $\xi$ upper-bound the normed difference in gradients between the weighted coreset and full dataset. Assume that the network $f(\bm{W}, \bm{X})$ is Lipschitz in $\bm{W}$, $\bm{X}$ with Lipschitz constant L and L' respectively, and $\bar{L} = \max\{L,L'\}$. Let $G_0$ the gradient over the full dataset at initialization, $\sigma_{\max}$ the maximum Jacobian singular value at initialization. Choosing perturbation bound $\epsilon_0 \leq \frac{1}{\sigma_{\max}\sqrt{\bar{L}n}}$ where $\sigma_{\max}$ is the maximum singular value of the coreset Jacobian and $n$ is the size of the original dataset, running SGD on the coreset and its augmentation using constant step size $\eta = \frac{\alpha}{\lambda\beta}$, we get the following convergence bound: 
\begin{align}
    \mathbb{E} [\| \nabla \Lc(\bm{W}^t,  \bm{X}_{c + c_{\aug}}) \|] \leq \frac{1}{\sqrt{\alpha}} \left( 1 - \frac{\alpha \eta}{2} \right)^{\frac{t}{2}} \left(2G_0 + 2\xi + \mathcal{O}\left(\frac{\bar{L}}{\sigma_{\max}}\right) \right),
\end{align}
where $\bm{X}_{c+c_{aug}}$ represents the dataset containing the (weighted) coreset and its augmentation.
\end{theorem}

\begin{proof}
As in the proof for Theorem \ref{theorem:alpha-pl}, we begin with the following inequality
\begin{align}
    \mathbb{E} [\| \nabla \Lc(\bm{W}^t,  \bm{X}_{c + c_{\aug}}) \|^2] &\leq \left( 1 - \frac{\alpha \eta}{2} \right)^{t} \mathcal{L}(\bm{W}^0, \bm{X}_{c+c_{\aug}}) \\
    &\leq \frac{1}{\alpha} \left( 1 - \frac{\alpha \eta}{2} \right)^{t} \| \nabla \Lc(\bm{W}^0,  \bm{X}_{c + c_{\aug}}) \|^2 
\end{align}

Thus, we can write
\begin{align}
    &\mathbb{E} [\| \nabla \Lc(\bm{W}^0,  \bm{X}_{c + c_{\aug}}) \|] \\
    &\leq \sqrt{\mathbb{E} [\| \nabla \Lc(\bm{W}^t,  \bm{X}_{c + c_{\aug}}) \|^2] } \\
    &\leq \frac{1}{\sqrt{\alpha}} \left( 1 - \frac{\alpha \eta}{2} \right)^{\frac{t}{2}} \| \nabla \Lc(\bm{W}^0,  \bm{X}_{c + c_{\aug}}) \| \\
    &\leq  \frac{1}{\sqrt{\alpha}} \left( 1 - \frac{\alpha \eta}{2} \right)^{\frac{t}{2}} \left( \| \nabla \Lc(\bm{W}^0,  \bm{X}_{c}) \|  + \| \nabla \Lc(\bm{W}^0,  \bm{X}_{c_{\aug}}) \| \right ) \\
    &\leq \frac{1}{\sqrt{\alpha}} \left( 1 - \frac{\alpha \eta}{2} \right)^{\frac{t}{2}} \left( G_0 + \xi + \| (\Jc(\bm{W}^0, \bm{X}_c) + \bm{E}) (\bm{} - f(\bm{W}^0, \bm{X}_c + \bm{\epsilon}) ) \| \right )
\end{align}
The rest of the proof is similar to that of Theorem $\ref{theorem:alpha-pl}$.
\end{proof}

\subsection{Lemma for eigenvalues of coreset}
The following Lemma characterizes the sum of eigenvalues of the NTK associated with the coreset.
\begin{lemma} %
\label{lemma:singular-values-coreset}
Let $\xi$ be an upper bound of the normed difference in gradient of the weighted coreset and the original dataset, i.e. for full data $\bm{X}$ and its corresponding coreset $\bm{X}_S$ with weights $\gamma_S$, and respective residuals $\bm{r}$, $\bm{r}_S$, we have the bound $\|\Jc^T(\W^t,\X)\bm{r}^t-\bm{\gamma}_S\Jc^T(\W^t,\X_S)\bm{r}_S^t\| \leq \xi$. Let $\{\lambda_i\}_{i=1}^k$ be the eigenvalues of the NTK associated with the coreset. Then we have that

\[
\sqrt{\sum_{i=1}^k \lambda_i} \geq  \frac{| \|\Jc^T(\W^t,\X)\bm{r}^t\| - \xi |}{\|\bm{r}^t_S\|}. 
\]
\end{lemma}

\begin{proof}
Let singular values of coreset Jacobian be $\sigma_i$. Let $\Jc^T(\W^t,\X)\bm{r}^t = \bm{\gamma}_S\Jc^T(\W^t,\X_S)\bm{r}_S^t + \xi_S$ where $\|\xi_S\| \leq \xi$.

Taking Frobenius norm, we get
\begin{align}
    &\|\bm{\gamma}_S\Jc^T(\W^t,\X_S)\bm{r}_S^t\| = \| \Jc^T(\W^t,\X)\bm{r}^t - \xi_S \| \\
    \Rightarrow & \|\bm{\gamma}_S\Jc^T(\W^t,\X_S)\| \|\bm{r}_S^t\| \geq \|\Jc^T(\W^t,\X)\bm{r}^t - \xi_S\| \\
    \Rightarrow & \|\bm{\gamma}_S\Jc^T(\W^t,\X_S)\| \geq  \frac{\|\Jc^T(\W^t,\X)\bm{r}^t - \xi_S\|}{\|\bm{r}_S^t\|} \\
    \Rightarrow & \sqrt{\sum_{i=1}^s \sigma_i^2} \geq  \frac{\|\Jc^T(\W^t,\X)\bm{r}^t - \xi_S\|}{\|\bm{r}_S^t\|} \\
    \Rightarrow & \sqrt{\sum_{i=1}^s \lambda_i} \geq  \frac{\|\Jc^T(\W^t,\X)\bm{r}^t - \xi_S\|}{\|\bm{r}_S^t\|} \\
    \Rightarrow & \sqrt{\sum_{i=1}^s \lambda_i} \geq  \frac{| \| \Jc^T(\W^t,\X)\bm{r}^t \|- \xi|}{\|\bm{r}_S^t\|} \quad \text{ by reverse triangle inequality}
\end{align}
\end{proof}
We can make the following observations:
For overparameterized networks, with bounded activation functions and labels, e.g. softmax and one-hot encoding, the norm of the residual vector is bounded, and the gradient norm is likely to be much larger than residual, especially when dimension of gradient is large. In this case, the Jacobian matrix associated with small weighted coresets found by solving Eq. \eqref{eq:fl_min}, have large singular values.

\subsection{Augmentation as Linear Transformation: Linear Model Analysis}
\label{app:aug-linear-transform}
We introduce a simplified linear model to extend our theoretical analysis to augmentations modelled as linear transformation matrices $F$ applied to the original training data. These augmentations are also originally studied by \cite{wu2020generalization}. In this section, we specifically study the effect of these augmentations using a linear model when applied to coresets.

\begin{lemma} [Augmented coreset gradient bounds: Linear]
\label{lemma:linear}
Let $f$ be a simple linear model with weights $\bm{W} \in \mathbb{R}^{d \times C}$ where $f(\bm{W}, \bm{x}_i) = \bm{W}^T\bm{x}_i$, trained on mean squared loss function $\Lc$.
Let $F \in \mathbb{R}^{d \times d}$ be a common linear augmentation matrix with norm $\|F\|$ with augmentation $\bm{x}_i^{\aug}$ given by $F\bm{x}_i$. Let coreset be of size $k$ and full dataset be of size $n$. Further assume that the predicted label of $\bm{x}_i$ and its augmentation $F\bm{x}_i$ are sufficiently close, i.e. there exists $\omega$ such that $\bm{W}^T(F \bm{x}_i)$ = $\bm{W}^T\bm{x}_i + z_i$, $\|z_i\| \leq \omega \ \forall i$. Let $\xi$ upper-bound the normed difference in gradients between the weighted coreset and full dataset. Then, the normed difference between the gradient of the augmented full data and augmented coreset is given by 
\[
\| \sum_{i \in V} \nabla \Lc(\bm{W}, \bm{x}_i^{\aug}) - \sum_{j=1}^k \gamma_{s_j} \nabla \Lc(\bm{W}, \bm{x}_{s_j}^{\aug}) \| \leq \|F\|(\xi + \sqrt{d} n \omega)
\]
for some (small) constant $\xi$.
\end{lemma}

\begin{proof}
By our assumption, we can begin with, 
\begin{equation}
\label{coreset-eqn}
\| \sum_{i \in V} \nabla \mathcal{L}(\bm{W}, \bm{x}_i) - \sum_{j=1}^k \gamma_{s_j} \nabla \mathcal{L}(\bm{W}, \bm{x}_{s_j}) \| \leq \xi
\end{equation}

Furthermore, by \cite{mirzasoleiman2020coresets}, we know that sum of the coreset weights $\gamma_{s_j}$ is given by
$\sum_{j=1}^{k=1} \gamma_{s_j} \leq n$.

Hence, 
\begin{align}
    &\| \sum_{i \in V} \nabla \mathcal{L}(\bm{W}, \bm{x}^{\aug}_i) - \sum_{j=1}^k \gamma_{s_j} \nabla \mathcal{L}(\bm{W}, \bm{x}^{\aug}_{s_j}) \| \\
    &= \| \sum_{i \in V} {(\mathcal{J}(\bm{W}, \bm{x}^{\aug}_i))}^T [\bm{W}^T(F\bm{x}_i) - y_i] - 
    \sum_{j=1}^k \gamma_{s_j} {(\mathcal{J}(\bm{W}, \bm{x}^{\aug}_{s_j}))}^T [\bm{W}^T(F\bm{x}_{s_j}) - y_{s_j}] \| \\
    &=   \| \sum_{i \in V} F\bm{x}_i  [\bm{W}^T(F\bm{x}_i) - y_i] - \sum_{j=1}^k \gamma_{s_j} F\bm{x}_{s_j} [\bm{W}^T(F\bm{x}_{s_j}) - y_{s_j}] \| \\
    &=  \|F \sum_{i \in V} \bm{x}_i (\bm{W}^T\bm{x}_i - y_i) - F \sum_{j=1}^k \gamma_{s_j} \bm{x}_{s_j} (\bm{W}^T\bm{x}_{s_j}  + z_i - y_{s_j}) \| \\
    &=  \| F \sum_{i \in V}\nabla L(\bm{W}, \bm{x}_i) - F \sum_{j=1}^k \gamma_{s_j} \nabla L(\bm{W}, \bm{x}_{s_j}) -F\sum_{j=1}^k \gamma_{s_j}\bm{x}_{s_j}z_{s_j}\| \\
    &\leq  \|F\| \| \sum_{i \in V}\nabla L(\bm{W}, \bm{x}_i) - \sum_{j=1}^k \gamma_{s_j} \nabla L(\bm{W}, \bm{x}_{s_j}) \| + \|F\| \|\sum_{j=1}^k \gamma_{s_j}\bm{x}_{s_j}z_{s_j}\|\\
    &\leq \|F\|\xi + \sqrt{d} \|F\|n \omega \\
    &= \|F\|(\xi + \sqrt{d} n \omega)
\end{align}
\end{proof}

\begin{corollary}
\label{corollary:linear}
In the simplified linear case above, the difference in gradients of the full training data with its augmentations ($\nabla \Lc(\bm{W}, \bm{X}_{f + aug})$) and gradients of the coreset with its augmentations ($\nabla \Lc(\bm{W}, \bm{X}_{c + c_{aug}})$) can be bounded by
\[
\| \nabla \Lc(\bm{W}, \bm{X}_{f + aug}) - \nabla \Lc(\bm{W}, \bm{X}_{c + c_{aug}}) \| \leq (\|F\|+1)\xi + \sqrt{d} \|F\|n \omega
\]
\end{corollary}

\begin{proof}
Applying Eq. \eqref{coreset-eqn} and Lemma \ref{lemma:linear}, we obtain

\begin{align}
    &\| \nabla \Lc(\bm{W}, \bm{X}_{f + aug}) - \nabla \Lc(\bm{W}, \bm{X}_{c + c_{aug}}) \| \\ 
    &= \| (\nabla \Lc(\bm{W}, \bm{X}_{f}) + \nabla \Lc(\bm{W}, \bm{X}_{aug})) - (\nabla \Lc(\bm{W}, \bm{X}_{c}) + \nabla \Lc(\bm{W}, \bm{X}_{c_{aug}})) \|  \\
    &= \| (\nabla \Lc(\bm{W}, \bm{X}_{f}) - \nabla \Lc(\bm{W}, \bm{X}_{c})) + (\nabla \Lc(\bm{W}, \bm{X}_{aug}) - \nabla \Lc(\bm{W}, \bm{X}_{c_{aug}})) \| \\
    &\leq \| (\nabla \Lc(\bm{W}, \bm{X}_{f}) - \nabla \Lc(\bm{W}, \bm{X}_{c}))\| + \|(\nabla \Lc(\bm{W}, \bm{X}_{aug}) - \nabla \Lc(\bm{W}, \bm{X}_{c_{aug}})) \| \\
    &\leq \xi + \|F\|(\xi + \sqrt{d} n \omega) \\
     &= (\|F\|+1)\xi + \sqrt{d} \|F\|n \omega
\end{align}
\end{proof}

\begin{theorem} [Convergence of linear model]
\label{theorem:linear}
Let $f$ be a linear model with weights $\bm{W}$ and augmentation be represented by the common linear transformation $F$. Let $\mathcal{L}_{i}$ be $\beta$-smooth, $\mathcal{L}$ be $\lambda$-smooth and satisfy the $\alpha$-PL condition, that is for $\alpha > 0$, ${\| \nabla \mathcal{L}(\bm{W}, \bm{X}) \|}^2 \geq \alpha \mathcal{L}(\bm{W}, \bm{X})$ for all weights $\bm{W}$. Let $\xi$ upper-bound the normed difference in gradients between the weighted coreset and full dataset and $\omega$ bound $\bm{W}^T(F\bm{x}_i) = \bm{W}^T\bm{x}_i + z_i$, $\|z_i\| \leq \omega \ \forall i$. Let $G_0'$ be the gradient over the full dataset and its augmentations at initialization. Then, running SGD on the size $k$ coreset with its augmentation using constant step size $\eta = \frac{\alpha}{\lambda\beta}$, we get the following convergence bound:
\[
\mathbb{E} [\| \nabla \mathcal{L}(\bm{W}^t, \bm{X}_{c + c_{\aug}}) \|] \leq \frac{1}{\sqrt{\alpha}} \left( 1 - \frac{\alpha \eta}{2} \right)^{\frac{t}{2}} \left( G_0' + (\|F\|+1)\xi + \sqrt{d} \|F\|n \omega \right )
\]
\end{theorem}
\begin{proof}
From \cite{bassily2018exponential}, we have
\begin{align}
    \mathbb{E} [\| \nabla \mathcal{L}(\bm{W}^t, \bm{X}_{c + c_{\aug}}) \|^2] &\leq \left( 1 - \frac{\alpha \eta}{2} \right)^{t} \mathcal{L}(\bm{W}^0, \bm{X}_{c + c_{\aug}}) \\
    &\leq \frac{1}{\alpha} \left( 1 - \frac{\alpha \eta}{2} \right)^{t} \| \nabla \mathcal{L}(\bm{W}^0, \bm{X}_{c + c_{\aug}}) \|^2 \\
\end{align}
Using Jensen's inequality, we have
\begin{align}
    &\mathbb{E} [\| \nabla \mathcal{L}(\bm{W}^t, \bm{X}_{c + c_{\aug}}) \|] \\
    &\leq \sqrt{\mathbb{E} [\| \nabla \mathcal{L}(\bm{W}^t, \bm{X}_{c + c_{\aug}}) \|^2] } \\
    &\leq \frac{1}{\sqrt{\alpha}} \left( 1 - \frac{\alpha \eta}{2} \right)^{\frac{t}{2}} \| \nabla \mathcal{L}(\bm{W}^0, \bm{X}_{c + c_{\aug}}) \| \\
    &\leq  \frac{1}{\sqrt{\alpha}} \left( 1 - \frac{\alpha \eta}{2} \right)^{\frac{t}{2}} \left( G_0' + (\|F\|+1)\xi + \sqrt{d} \|F\|n \omega \right )
\end{align}
where the last inequality follows from applying Corollary \ref{corollary:linear}.
\end{proof}

\section{Singular spectrum analysis}
\label{app:singular-spectrum-analysis}

\subsection{Experiment details}
\label{app:singular-spectrum-exp-details}
We generate singular spectrum plots for both MNIST and CIFAR10 datasets in \cref{fig:spectrum,fig:singular-values-and-vectors-large}. Due to the computational infeasbility of computing the network Jacobian for the full datasets in deep network settings, we instead construct and use a reduced version of these datasets by uniformly select 900 images from the first 3 classes. For our experiments on MNIST, we pretrain a MLP model with 1 hidden layer for 15 epochs. For our experiments on CIFAR10, we pretrain a ResNet20 model for 15 epochs. We then compute the singular spectrums for augmented and non-augmented data based on these pretrained networks.

Since it is difficult to perform a one-to-one matching of singular values produced from augmented and non-augmented datasets, we instead bin our singular values into 30 separate and uniformly distributed bins each containing the same number of singular values. To measure perturbation to singular values resulted from augmentation, we compute the mean difference between each bin. On the other hand, to measure perturbation to singular vectors, we compute mean subspace angle between the singular subspace spanned by singular vectors in each bin.

\subsection{Real-world strong augmentations}
\label{app:singular-spectrum-realworld}
We study the effects of real-world, unbounded augmentations on the singular spectrum of the network Jacobian. In particular, in additional to the plots in the main paper, we show the effect of strong augmentations through (1) random rotation (up to $30\circ$ and AutoAugment \citep{cubuk2019autoaugment} for MNIST and (2) random horizontal flips/random crops and AutoAugment for CIFAR10. The policies implemented by AutoAugment include translations, shearing, as well as contrast and brightness transforms. We study the effects of these augmentations on the singular spectrum in \cref{fig:singular-values-and-vectors-large}. Despite these augmentations being unbounded transformations, we note that the results of our theory still holds. In particular, it can be observed that data augmentation increases smaller singular values relatively more with
a higher probability. On the other hand, data augmentation
affects the prominent singular vectors of the Jacobian to a
smaller extent, and preserves the prominent directions. As such, our argument empirically extends to real-world, unbounded label-invariant transformations characteristic of strong augmentations.

\begin{figure}[h!]
   \centering
    \begin{subfigure}[MNIST $\epsilon_0=8$ - Values \label{fig:app_mnist_noise8_relative_values}]{
    \includegraphics[width=.225\textwidth,trim=10mm 0 0mm 10mm]{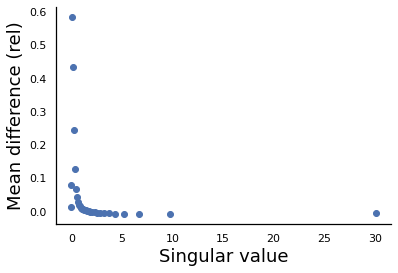}}
    \end{subfigure}\hspace{0mm}
    \begin{subfigure}[MNIST $\epsilon_0=8$ - Vectors \label{fig:app_mnist_noise8_vectorspace}]{
    \includegraphics[width=.225\textwidth,trim=10mm 0 0mm 10mm]{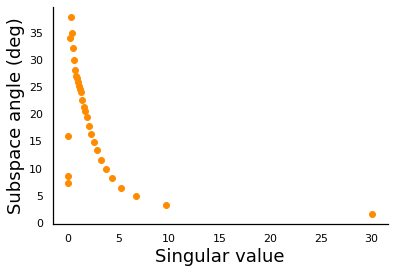}}
    \end{subfigure}
    \begin{subfigure}[CIFAR10 $\epsilon_0=8$ - Values \label{fig:app_cifar_noise8_relative_values}]{
    \includegraphics[width=.225\textwidth,trim=10mm 0 0mm 10mm]{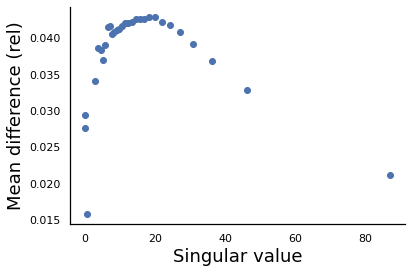}}
    \end{subfigure}\hspace{0mm}
    \begin{subfigure}[CIFAR10 $\epsilon_0=8$ - Vectors \label{fig:app_cifar_noise8_vectorspace}]{
    \includegraphics[width=.225\textwidth,trim=10mm 0 0mm 10mm]{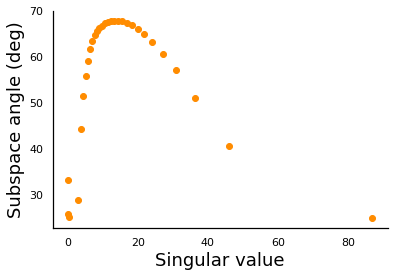}}
    \end{subfigure}
    
    \centering
    \begin{subfigure}[MNIST - $\epsilon_0=16$ - Values \label{fig:app_mnist_noise16_relative_values}]{
    \includegraphics[width=.225\textwidth,trim=10mm 0 0mm 10mm]{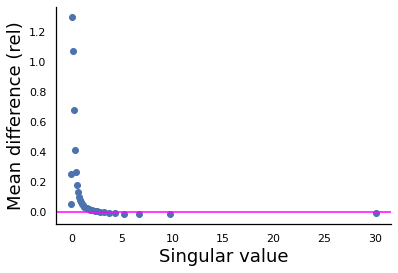}}
    \end{subfigure}\hspace{0mm}
    \begin{subfigure}[MNIST $\epsilon_0=16$ - Vectors \label{fig:app_mnist_noise16_vectorspace}]{
    \includegraphics[width=.225\textwidth,trim=10mm 0 0mm 10mm]{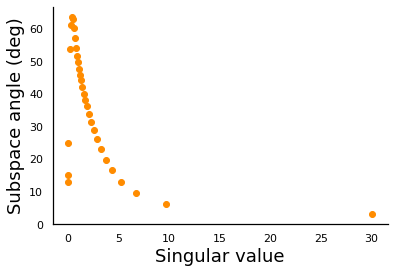}}
    \end{subfigure}
    \begin{subfigure}[CIFAR10 $\epsilon_0=16$ - Values \label{fig:app_cifar_noise16_relative_values}]{
    \includegraphics[width=.225\textwidth,trim=10mm 0 0mm 10mm]{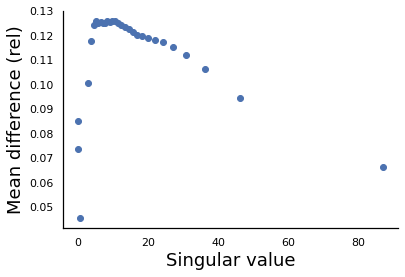}}
    \end{subfigure}\hspace{0mm}
    \begin{subfigure}[CIFAR10 $\epsilon_0=16$ - Vectors \label{fig:app_cifar_noise16_vectorspace}]{
    \includegraphics[width=.225\textwidth,trim=10mm 0 0mm 10mm]{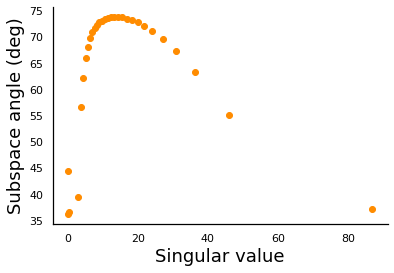}}
    \end{subfigure}
    
    \centering
    \begin{subfigure}[MNIST Rotate - Values \label{fig:mnist_rotate30_relative_values}]{
    \includegraphics[width=.225\textwidth,trim=10mm 0 0mm 10mm]{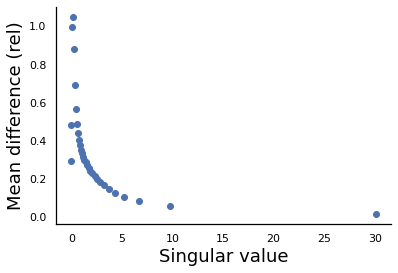}}
    \end{subfigure}\hspace{0mm}
    \begin{subfigure}[MNIST Rotate - Vectors \label{fig:mnist_rotate30_vectorspace}]{
    \includegraphics[width=.225\textwidth,trim=10mm 0 0mm 10mm]{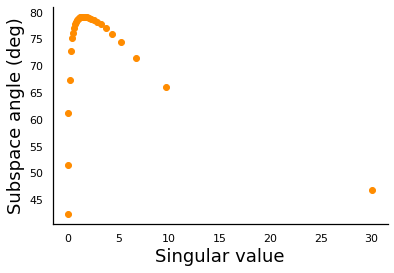}}
    \end{subfigure}
    \begin{subfigure}[CIFAR10 Flip-Crop - Values \label{fig:cifar_flipcrop_relative_values}]{
    \includegraphics[width=.225\textwidth,trim=10mm 0 0mm 10mm]{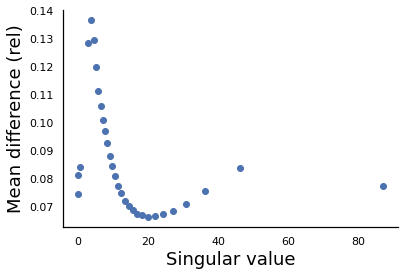}}
    \end{subfigure}\hspace{0mm}
    \begin{subfigure}[CIFAR10 Flip-Crop - Vectors \label{fig:cifar_flipcrop_vectorspace}]{
    \includegraphics[width=.225\textwidth,trim=10mm 0 0mm 10mm]{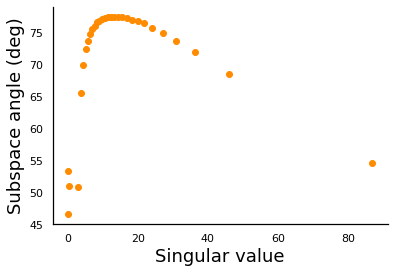}}
    \end{subfigure}
    
    \centering
    \begin{subfigure}[MNIST AutoAugment - Values \label{fig:mnist_autoaugment_relative_values}]{
    \includegraphics[width=.225\textwidth,trim=10mm 0 0mm 10mm]{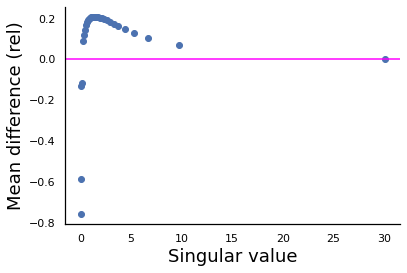}}
    \end{subfigure}\hspace{0mm}
    \begin{subfigure}[MNIST AutoAugment - Vectors \label{fig:mnist_autoaugment_vectorspace}]{
    \includegraphics[width=.225\textwidth,trim=10mm 0 0mm 10mm]{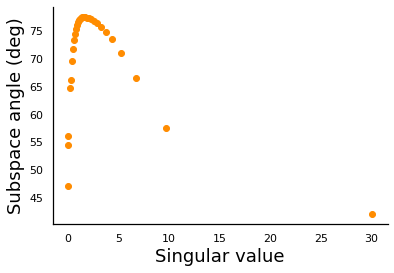}}
    \end{subfigure}
    \begin{subfigure}[CIFAR10 AutoAugment - Values \label{fig:cifar_autoaugment_relative_values}]{
    \includegraphics[width=.225\textwidth,trim=10mm 0 0mm 10mm]{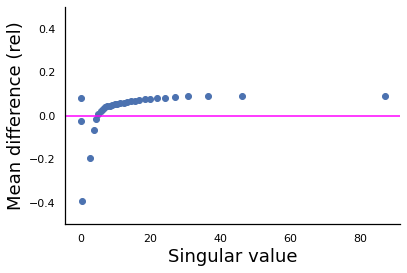}}
    \end{subfigure}\hspace{0mm}
    \begin{subfigure}[CIFAR10 AutoAugment - Vectors \label{fig:cifar_autoaugment_vectorspace}]{
    \includegraphics[width=.225\textwidth,trim=10mm 0 0mm 10mm]{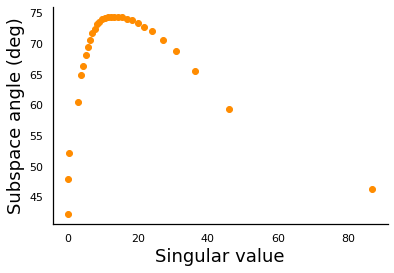}}
    \end{subfigure}
    
    \centering
    \begin{subfigure}[MNIST Rotate + AutoAugment - Values \label{fig:mnist_rotate30_autoaugment_relative_values}]{
    \includegraphics[width=.225\textwidth,trim=10mm 0 0mm 10mm]{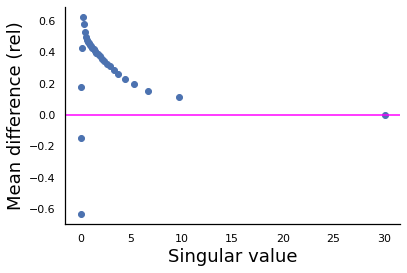}}
    \end{subfigure}\hspace{0mm}
    \begin{subfigure}[MNIST Rotate + AutoAugment - Vectors \label{fig:mnist_rotate30-autoaugment_vectorspace}]{
    \includegraphics[width=.225\textwidth,trim=10mm 0 0mm 10mm]{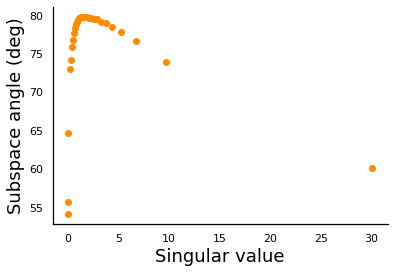}}
    \end{subfigure}
    \begin{subfigure}[CIFAR10 Flip + Crop + AutoAugment - Values \label{fig:cifar_flip-crop-autoaugment_relative_values}]{
    \includegraphics[width=.225\textwidth,trim=10mm 0 0mm 10mm]{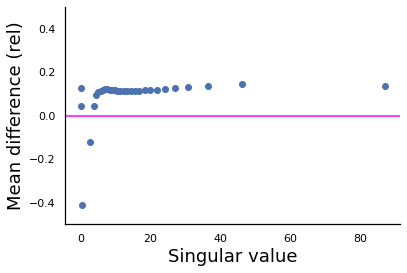}}
    \end{subfigure}\hspace{0mm}
    \begin{subfigure}[CIFAR10 Flip + Crop + AutoAugment - Vectors \label{fig:cifar_flip-crop-autoaugment_vectorspace}]{
    \includegraphics[width=.225\textwidth,trim=10mm 0 0mm 10mm]{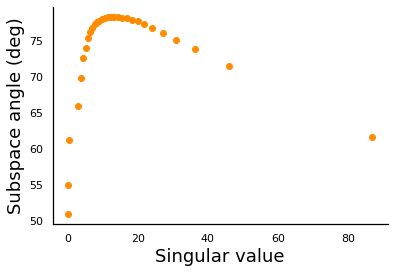}}
    \end{subfigure}
    
    \caption{Difference in mean singular values (Cols 1 \& 3) between augmented and non-augmented data and mean angular difference (Cols 2 \& 4) between subspaces spanned by singular vectors for augmented and non-augmented data.}
    \label{fig:singular-values-and-vectors-large}
\end{figure}

\section{Experiment Setup and Additional Experiments}
\label{app:extra-experiments}

\subsection{Experiment setup}
For all experiments, we train using SGD with 0.9 momentum and learning rate decay. For experiments on CIFAR10 and variants/ResNet20, we train for 200 epochs, for Caltech256 (ImageNet pretrained)/ ResNet18, we trained for 40 epochs starting at learning rate  $0.001$ and batch size 64. We also report results for Caltech256 without ImageNet pretraining in \secref{sec:caltech256-from-scratch}, where we train for 400 epochs to ensure convergence with a starting learning rate of  $0.05$ and batch size 64. For experiments on ImageNet/ResNet50 and TinyImageNet/ResNet50, we use the standard 90 epoch learning schedule starting at learning rate of 0.1 and batch size 64.

\textbf{Data and augmentation.} We apply our method to training ResNet20 and Wide-ResNet-28-10 on CIFAR10, and ResNet32 on CIFAR10-IMB (Long-Tailed CIFAR10 with Imbalance factor of 100 following \citep{kim2020adjusting}) and SVHN datasets. 
We train Caltech256 \cite{griffin2007caltech} on ImageNet-pretrained ResNet18, and include experiments with random initialization in Appendix D. TinyImageNet and ImageNet are trained on with ResNet50. We use \citep{wu2020generalization} for CIFAR10/SVHN, and AutoAugment \cite{cubuk2019autoaugment} for Caltech256, TinyImageNet, and ImageNet as the strong augmentation method.
 Note that we append strong augmentations rather than apply them in-place, which we show to be more effective in Appendix D.
All results are averaged over 5 runs using an Nvidia A40 GPU. 

\subsection{Experiment setup}
For all experiments, we train using SGD with 0.9 momentum and learning rate decay. We also set weight decay as For experiments on CIFAR10 and variants/ResNet20, we train for 200 epochs, for Caltech256 (ImageNet pretrained)/ ResNet18, we trained for 40 epochs starting at learning rate  $0.001$ and batch size 64. We also report results for Caltech256 without ImageNet pretraining in \secref{sec:caltech256-from-scratch}, where we train for 400 epochs to ensure convergence with a starting learning rate of  $0.05$ and batch size 64. For experiments on ImageNet/ResNet50 and TinyImageNet/ResNet50, we use the standard 90 epoch learning schedule starting at learning rate of 0.1 and batch size 64.

\textbf{Data and augmentation.} We apply our method to training ResNet20 and Wide-ResNet-28-10 on CIFAR10, and ResNet32 on CIFAR10-IMB (Long-Tailed CIFAR10 with Imbalance factor of 100 following \citep{kim2020adjusting}) and SVHN datasets. 
We train Caltech256 \cite{griffin2007caltech} on ImageNet-pretrained ResNet18, and include experiments with random initialization in Appendix D. TinyImageNet and ImageNet are trained on with ResNet50. We use \citep{wu2020generalization} for CIFAR10/SVHN, and AutoAugment \cite{cubuk2019autoaugment} for Caltech256, TinyImageNet, and ImageNet as the strong augmentation method.
 Note that we append strong augmentations rather than apply them in-place, which we show to be more effective in Appendix D.
All results are averaged over 5 runs using an Nvidia A40 GPU. 

\subsection{Full Results for Table \ref{tab:results-main}}
This section contains full experiment results including standard deviations and the full augmentation benchmark for Table $\ref{tab:results-main}$. Augmenting coresets of size 10\% achieves 51\% of the improvement obtained from augmentation of the full data and further enjoys a 6x speedup in training time on CIFAR10. This speedup becomes more significant when using strong augmentation techniques which are mostly computationally demanding, especially when applied to the entire dataset.

\begin{table}[h]
\small
\centering
\caption{Supplementary table for Table \ref{tab:results-main} - Test accuracy on CIFAR10 + ResNet20, SVHN + ResNet32, CIFAR10-Imbalanced + ResNet32 including standard deviation errors and full dataset augmentation accuracy. 
}
\begin{tabular}{*{5}c}
\toprule
Method & 
Size &
CIFAR10 &
CIFAR10-IMB &
SVHN \\
\midrule
None &
$0\%$ &
$89.46 \pm 0.17\%$ &
$87.08 \pm 0.50\%$ &
$95.676 \pm 0.108\%$ 
\\
\cmidrule(lr){1-5}
&
$5\%$ &
$90.34 \pm 0.18\%$ &
$88.48 \pm 0.25\%$ &
$95.760 \pm 0.084\% $
\\
Random &
$10\%$ &
$91.07 \pm 0.13\%$ &
$89.52 \pm 0.15\%$ &
$96.187 \pm 0.112 \%$
\\
&
$30\%$&
$92.11 \pm 0.12\%$ &
$91.11 \pm 0.18\%$ &
$96.569 \pm 0.073 \%$
\\
\cmidrule(lr){1-5}
&
$5\%$ &
$90.79 \pm 0.19\%$ &
$88.77 \pm 0.35\%$ &
$\bm{96.165 \pm 0.108 \%}$
\\
Max-Loss &
$10\%$ &
$91.39 \pm 0.08\%$ &
$89.22 \pm 0.48\%$ &
$\bm{96.370 \pm 0.076\%}$
\\
&
$30\%$ &
$92.43 \pm 0.07\%$ &
$91.11\pm 0.25\%$ &
$96.735 \pm 0.068\%$
\\
\cmidrule(lr){1-5}
&
$5\%$ &
$\bm{90.87 \pm 0.05\%}$ &
$\bm{89.10 \pm 0.41\%}$ &
$96.121 \pm 0.055\%$
\\
Coreset &
$10\%$ &
$\bm{91.54 \pm 0.19\%}$ &
$\bm{89.75 \pm 0.52\%}$ &
$96.354 \pm 0.091 \%$
\\
&
$30\%$ &
$\bm{92.49 \pm 0.15\%}$ &
$\bm{91.12 \pm 0.26\%}$ &
$\bm{96.791 \pm 0.051\%}$
\\
\cmidrule(lr){1-5}
All &
$100\%$ &
$93.50 \pm 0.25\%$ &
$92.48 \pm 0.34\%$ &
$97.068 \pm 0.030\%$
\\
\bottomrule
\end{tabular}
\vspace{5mm}
\label{tab:results-main-app}
\end{table}

\subsection{Supplementary results for \tabref{tab:results-train-subset}}
\label{app:results-train-subset-include-weak-aug-only}
\begin{table*}[h]
 \caption{Supplementary results for \tabref{tab:results-train-subset}. Training ResNet20 (R20) and WideResnet-28-10 (W2810) on CIFAR10 (C10) using small subsets, and ResNet18 (R18) on Caltech256 (Cal). 
 } %
    \centering
    \footnotesize
    \resizebox{\textwidth}{!}
    {%
      \begin{tabular}{cccccc}
  \toprule
  \multicolumn{1}{c}{Model/Dataset}
  & \multicolumn{1}{c}{Subset}
  & \multicolumn{2}{c}{Random}
  & \multicolumn{2}{c}{Ours} \\
     \cmidrule(lr){3-4}
     \cmidrule(lr){5-6}
  &
  & \multicolumn{1}{c}{Weak Aug.}
  & \multicolumn{1}{c}{Strong Aug.}
  & \multicolumn{1}{c}{Weak Aug.}
  & \multicolumn{1}{c}{Strong Aug.}
  \\
  \midrule
     \multirow{4}{*}{C10/R20} 
     & 0.1\% (5)
     & $31.7 \pm 3.2$
     & $33.5 \pm 2.7$
     & $29.6 \pm 3.8$
     & \bm{$37.8 \pm 4.5$}
     \\   
     & 0.2\% (10)
     & $35.9 \pm 2.1$
     & $42.7 \pm 3.9$
     & $33.6 \pm 3.2$
     & \bm{$45.1 \pm 2.3$}
     \\   
     & 0.5\% (25)
     & $51.1 \pm 2.3$
     & $58.7 \pm 1.3$
     & $55.8 \pm 3.1$
     & \bm{$63.9 \pm 2.1$}
     \\   
     & 1\% (50)
     & $66.2 \pm 1.0$
     & $74.4 \pm 0.8$
     & $65.9 \pm 4.0$
     & \bm{$74.7 \pm 1.1$}
     \\
     \midrule
     \multirow{1}{*}{C10/W2810} 
     & 1\% (50)
     & $61.3 \pm 2.4$
     & $57.7 \pm 0.8$
     & $59.9 \pm 2.4$
     & \bm{$62.1 \pm 3.1$}
     \\
     \midrule
     \multirow{1}{*}{Cal/R18} 
     & 5\% (3)
     & $24.8 \pm 1.5$
     & $41.5 \pm 0.5$
     & $33.8 \pm 1.7$
     & \bm{$52.7 \pm 1.2$}
     \\
     & 10\% (6)
     & $49.5 \pm 0.6$
     & $61.8 \pm 0.8$
     & $55.7 \pm 0.3$
     & \bm{$65.4 \pm 0.3$}
     \\
     & 20\% (12)
     & $66.6 \pm 0.2$
     & $72.5 \pm 0.1$
     & $67.5 \pm 0.3$
     & \bm{$73.1 \pm 0.1$}     \\
     & 30\% (18)
     & $72.0 \pm 0.1$
     & $75.7 \pm 0.2$
     & $71.9 \pm 0.2$
     & \bm{$76.3 \pm 0.2$}
     \\
     & 40\% (24)
     & $74.6 \pm 0.3$
     & $77.6 \pm 0.4$
     & $74.2 \pm 0.4$
     & \bm{$77.7 \pm 0.5$}
     \\
     & 50\% (30)
     & $76.1 \pm 0.5$
     & $78.5 \pm 0.3$
     & $76.1 \pm 0.1$
     & \bm{$78.9 \pm 0.2$}
     \\
     \bottomrule
 \end{tabular}
 }
\vspace{-3mm}
\end{table*}

\subsection{Training dynamics vs generalization}
\label{appendix:training-dynamics-vs-generalization}
\cref{fig:train-loss-val-acc} demonstrates the relationship between training loss and validation accuracy resulted from data augmentation. While training loss of augmented datasets do not decrease as quickly as non-augmented datasets, generalization performance (shown by val. acc.) improves.

\begin{figure}[H]
    \centering
    \begin{subfigure}{
    \includegraphics[width=0.32\textwidth,trim=15mm 0 0mm 10mm]{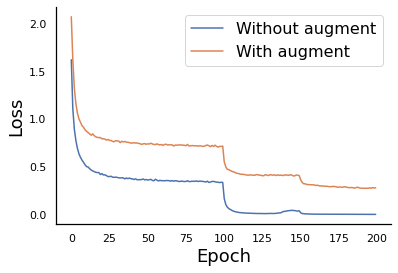}}
    \end{subfigure}
    \begin{subfigure}{
    \includegraphics[width=0.30\textwidth,trim=10mm 0 12mm 10mm]{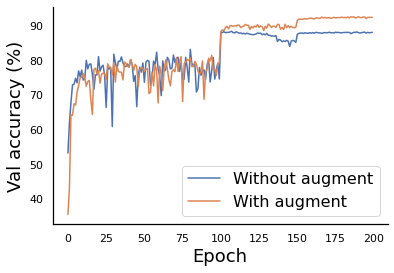}}
    \end{subfigure}
    \caption{Training loss vs validation accuracy of CIFAR10/ResNet20 using AutoAugment.
    }
    \label{fig:train-loss-val-acc}
\end{figure}

\subsection{Augmentations applied through appending vs in-place}\label{sec:append}
Our experiments on Caltech256/ResNet18/AutoAugment ($\!R\!=\!\!5$)
show that even for cheaper strong augmentation methods (AutoAugment), 
while in-place augmentation may decrease the performance,
appending Random (R) and Coresets (C) augmentations (Append) outperforms in-place augmentation of 2x data points (In-place 2x) for various subset sizes.
\begin{table}[H]
    \label{tab:aug-inplace}
    \vspace{-3mm}
    \centering
    \footnotesize
    \caption{Caltech256/AutoAugment in-place vs. appending for Caltech256.}
    \setlength{\tabcolsep}{3pt}
    \begin{tabular}{ccccc}
    \toprule
    & No Aug. & In-place & In-place (2x) & Append \\
    \midrule
    C5\% & 33.8\% & 26.4\% & 48.2\% & \textbf{52.7\%} \\
    R5\% & 24.8\% & 17.4\% & 40.2\% & \textbf{41.5\%} \\
    C10\%  & 55.7\% & 48.2\% & 62.8\% & \textbf{65.4\%} \\
    R10\% & 50.6\% & 40.2\% & \textbf{62.0\%} & 61.8\% \\
    C30\% & 71.9\% & 68.8\% & 74.9\% & \textbf{76.3\%} \\
    R30\% & 72.0\% & 68.7\% & 75.1\% & \textbf{75.7\%} \\
    \bottomrule
    \end{tabular}
    \vspace{-5mm}
\end{table}

\subsection{Speed-up measurements}
\label{app:speed-up-measurements}
We measure the improvement in training time in the case of training on full data and augmenting subsets of various sizes. While our method yields similar or slightly lower speed-up to the max-loss policy and random approach respectively, our resulting accuracy outperforms these two approaches. We show this in \figref{fig:coreset-maxloss-speedup-acc-30pc}. For example, for SVHN/Resnet32 using $30\%$ coresets, we sacrifice $11\%$ of the speed-up to obtain an additional $24.8\%$ of the gain in accuracy from full data augmentation when compared to a random subset of the same size. We show the speed-up obtained for our method and various subset sizes in \tabref{tab:results-speedup}, and provide wall-clock times for our method in \tabref{tab:wall-clock-time}.

\begin{table}[h]
 \caption{Speedup on CIFAR10 + ResNet20 (C10/R20), SVHN + ResNet32 (SVHN/R32). 
 }
    \centering
     \resizebox{\columnwidth}{!}{%
      \begin{tabular}{*{10}c}
  \toprule
  \multicolumn{1}{c}{Dataset}
  & \multicolumn{1}{c}{Full Aug.}
  & \multicolumn{6}{c}{Ours} 
  & \multicolumn{1}{c}{Max loss.} 
  & \multicolumn{1}{c}{Random.}\\
     \cmidrule(lr){3-8}
     \cmidrule(lr){9-9}
     \cmidrule(lr){10-10}
  & \multicolumn{1}{c}{$100\%$}
  & \multicolumn{1}{c}{$5\%$}
  & \multicolumn{1}{c}{$10\%$}
  & \multicolumn{1}{c}{$15\%$}
  & \multicolumn{1}{c}{$20\%$}
  & \multicolumn{1}{c}{$25\%$}
  & \multicolumn{1}{c}{$30\%$}
  & \multicolumn{1}{c}{$30\%$}
  & \multicolumn{1}{c}{$30\%$} \\
     \midrule
     C10 / R20
     & $1$x
     & $7.93$x
     & $6.31$x
     & $4.46$x
     & $4.27$x
     & $3.41$x
     & $3.43$x
     & $3.48$x
     & $4.03$x
     \\   
     SVHN / R32
    & $1$x
     & $5.35$x
     & $3.93$x
     & $3.40$x
     & $2.80$x
     & $2.49$x
     & $2.18$x
     & $2.21$x
     & $2.43$x
     \\   
     \bottomrule
 \end{tabular}
 }
\label{tab:results-speedup}
\end{table}

\begin{figure}[h]
\centering
\label{fig:coreset-maxloss-speedup-acc-30pc}
\begin{subfigure}[CIFAR10/ResNet20 \label{fig:coreset-vs-maxloss-points}]{
\includegraphics[width=0.21\textwidth,trim=10mm 0 12mm 10mm]{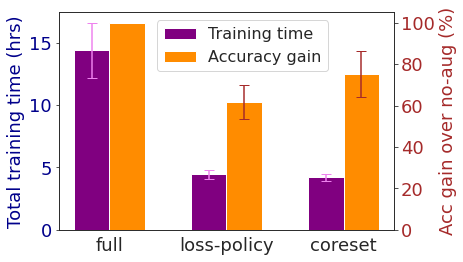}%
}%
\end{subfigure}
\begin{subfigure}[SVHN/Resnet32]{
\label{fig:svhn_original_max_loss_acc_train_time_0-30_subset}
\includegraphics[width=0.20\textwidth,trim=10mm 0 12mm 10mm]{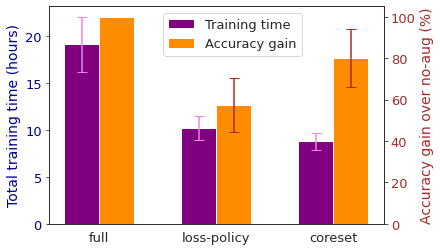}
\vspace{-1mm}
}\hspace{3mm}
\caption{Speedup/Accuracy of augmenting $30\%$ coresets compared to original max-loss policy for (a) ResNet20 trained on CIFAR10 and (b) ResNet32 trained on SVHN.}
\end{subfigure}\vspace{-3mm}
\end{figure}

\begin{table}[h]
    \vspace{-1mm}
    \label{tab:wall-clock-time}
    \centering
    \footnotesize
    \caption{Wall-clock times to find various sized coresets from \textit{all classes} of Caltech256 and TinyImagene at $\!1\!$ epoch. Note, training ResNet20/CIFAR10 with \cite{wu2020generalization} takes 14.4 hrs. In practice, coresets can be found in parallel ($p$ threads) from different classes, and selection happens every $\!R\!\!=\!5\!-\!15$ epochs. Hence, the numbers divide by $p \!\times\! R$.}
    \begin{tabular}{cccccc}
    \toprule
    \multicolumn{3}{c}{Caltech256} & \multicolumn{3}{c}{TinyImageNet} \\
    \cmidrule(lr){1-3}
        \cmidrule(lr){4-6}
        10\% & 30\% & 50\% & 10\% & 30\% & 50\% \\
    \midrule
    10.50s & 10.52s & 10.53s & 7.85s & 8.09s & 8.17s \\
    \bottomrule
    \end{tabular}
    \vspace{-5mm}
\end{table}

\subsection{End to end training on Caltech256}
\label{sec:caltech256-from-scratch}
As Caltech256 contains many classes and higher resolution images, training on smaller subset without pretraining has a low accuracy. %
Thus, many works (e.g. Achille et al., 2020) finetune from ImageNet pretrained initialization. However, we show that our results still hold even when training form scratch. We demonstrate our results in \tabref{tab:caltech-nopt}, where we train Caltech256 on ResNet50 without pretraining for 400 epochs, and with $R=40$, where our method consistently outperfoms random subsets for multiple subset sizes ($5\%$, $10\%$, $30\%$, $50\%$).
\begin{table}[h]
    \caption{Caltech256 (w/o pretraining) /ResNet50, 400 epochs, $R=40$}
    \label{tab:caltech-nopt}
    \centering
    \footnotesize
    \setlength{\tabcolsep}{3pt}
    \begin{tabular}{cccccccc}
    \toprule
    \multicolumn{4}{c}{Random} & \multicolumn{4}{c}{Ours} \\
        \cmidrule(lr){1-4}
        \cmidrule(lr){5-8}
        5\% & 10\% & 30\% & 50\% & 5\% & 10\% & 30\% & 50\% \\
    \midrule
    17.26 & 35.38 & 58.2 & 64.67 & \textbf{20.58} & \textbf{38.20} & \textbf{60.30} & \textbf{65.17}  \\
    \bottomrule
    \end{tabular}
\end{table}

\subsection{Training on full data and augmenting small subsets re-selected every epoch}
We apply our proposed method to select a new subset for augmentation every epoch (i.e. using $R=1$) and compare our results with other approaches using accuracy and percentage of data not selected (NS). 
We see that while the max-loss policy selects a small fraction of data points over and over and random uniformly selects all the data points, our approach successfully finds the smallest subset of data points that are the most crucial for data augmentation. Hence, it can achieve a superior accuracy than max-loss policy, while augmenting only slightly more examples. This confirms the data-efficiency of our approach.
This is especially evident when using coresets of size $0.2\%$. Furthermore, despite the random baseline using a significantly larger percentage of data, it is outperformed by our approach in both data-efficiency and accuracy. We emphasize that results in this table is different from that of Table $\ref{tab:results-main-app}$, as default augmentations on the full training data are performed once every $R=1$ epochs instead of every $R=20$ epochs. Since selecting subsets at every epoch can be computationally expensive, we only perform these experiments on small coresets and hence still enjoy good speedups compared to full data augmentation. This shows that our approach is still effective at very small subset sizes, hence can be computationally efficient even when subsets are re-selected every epoch.

\begin{table}[H]
    \centering
    \small
    \caption{Training on full data and selecting a new subset for augmentation every epoch ($R=1$).}
    {%
      \begin{tabular}{*{7}c}
      \toprule
      \multicolumn{1}{c}{Subset}
      & \multicolumn{2}{c}{Random} 
      & \multicolumn{2}{c}{Max-loss Policy} 
      & \multicolumn{2}{c}{Ours}
      \\
     \cmidrule(lr){2-3}
     \cmidrule(lr){4-5}
     \cmidrule(lr){6-7}
      & \multicolumn{1}{c}{Acc}
      & \multicolumn{1}{c}{NS (\%)}
      & \multicolumn{1}{c}{Acc}
      & \multicolumn{1}{c}{NS (\%)}
      & \multicolumn{1}{c}{Acc}
      & \multicolumn{1}{c}{NS (\%)}
      \\
      \midrule
      $0\%$ 
      & $91.96 \pm 0.12$
      & $-$ 
      & $91.96 \pm 0.12$
      & $-$
      & $91.96 \pm 0.12$
      & $-$
      \\
      $0.2\%$ 
      & $92.22 \pm 0.22$
      & $67.03 \pm 0.04$ 
      & $91.94 \pm 0.12$
      & $86.70 \pm 0.15$
      & $\bm{92.26} \pm 0.13$
      & $79.19 \pm 1.10$
      \\
      $0.5\%$
      & $92.06 \pm 0.17$
      & $36.70 \pm 0.18$
      & $92.20 \pm 0.13$
      & $76.80 \pm 0.31$
      & $\bm{92.27} \pm 0.08$
      & $63.23 \pm 0.35$
      \\
      \bottomrule
      \end{tabular}
      }
      \medskip
    \label{tab:results-small-subset}
\end{table}

\subsection{Additional visualizations for training on coresets and its augmentations - Measuring training dynamics over time}
We include additional visualizations in Figure \ref{fig:coreset_vs_data_subset_sizes} for training on coresets and its augmentations as supplementary plots to Figure \ref{fig:coreset_train_vs_not_selected} and Table \ref{tab:results-train-subset}. We plot metrics obtained during each point (epoch) of the training process based on percentage of data selected/used and test accuracy achieved. All metrics are averaged over 5 runs and obtained using $R=1$. These plots demonstrate that coreset augmentation approaches outperform random augmentation baselines throughout the training process. Furthermore, they show that augmentation of coresets result in a larger 
increase in test accuracy compared to augmentation of randomly selected training examples, especially for small subset sizes.

\begin{figure}[t]
    \centering
    \subfigure[CIFAR10 - 0.1\%]{
	\includegraphics[width=.28\textwidth,trim=10mm 0 12mm 10mm]{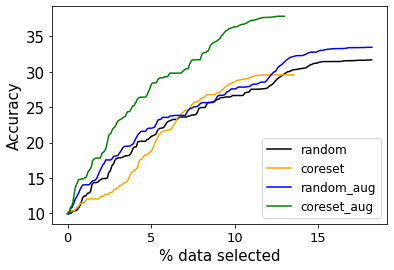}
    }\hspace{5mm}
    \subfigure[CIFAR10 - 0.2\%]{
	\includegraphics[width=.28\textwidth,trim=10mm 0 12mm 10mm]{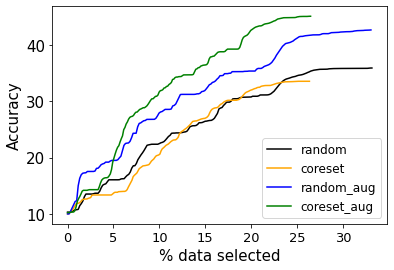}
    }\hspace{05mm}
    \subfigure[CIFAR10 - 0.5\%]{
	\includegraphics[width=.28\textwidth,trim=10mm 0 12mm 10mm]{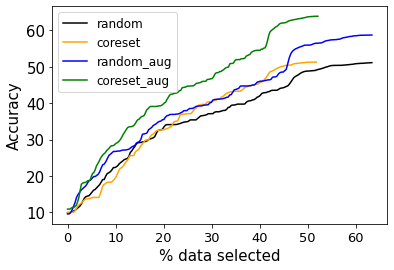}
    }\hspace{5mm} \\
    \vspace{5mm}
    \subfigure[CIFAR10 -  1\%]{
	\includegraphics[width=.28\textwidth,trim=10mm 0 12mm 10mm]{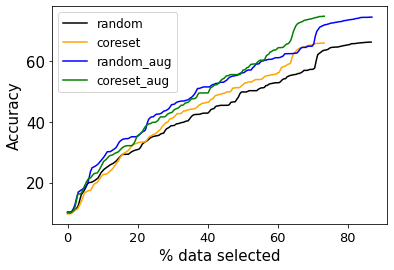}
    }\hspace{5mm}
    \subfigure[CIFAR10 - 5\%]{
	\includegraphics[width=.28\textwidth,trim=10mm 0 12mm 10mm]{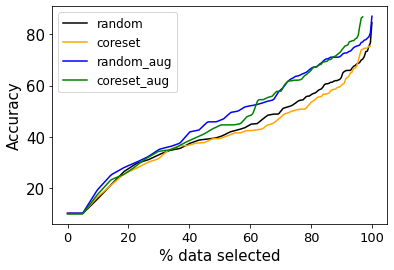}
    }\hspace{5mm}
    \caption{Supplementary plots for Figure \ref{fig:coreset_train_vs_not_selected}: Training on coreset and its augmentation compared to random baseline, measured using test accuracy against percentage of data used on CIFAR10 dataset across various subset sizes. Accuracy and percentage of data used are measured at every epoch and averaged over 5 runs.
    }
    \label{fig:coreset_vs_data_subset_sizes}
\end{figure}

\begin{figure}[!ht]
    \centering
    \includegraphics[width=0.31\textwidth,trim=10mm 0 12mm 10mm]{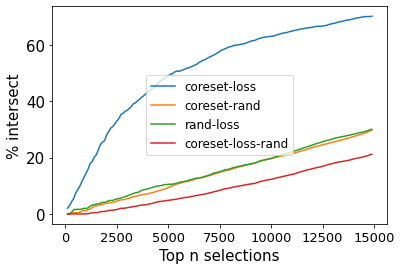}
    \caption{Intersection between max-loss and coresets in the top $N$ points selected aggregated across the entire training process. Here, we show the increasing overlap between max-loss and coreset points as $N$ grows.}
    \label{fig:coreset-maxloss-topn-intersect}
\end{figure}

\begin{figure}[!ht]
    \centering
    \includegraphics[width=0.8\textwidth,trim=10mm 0 12mm 10mm]{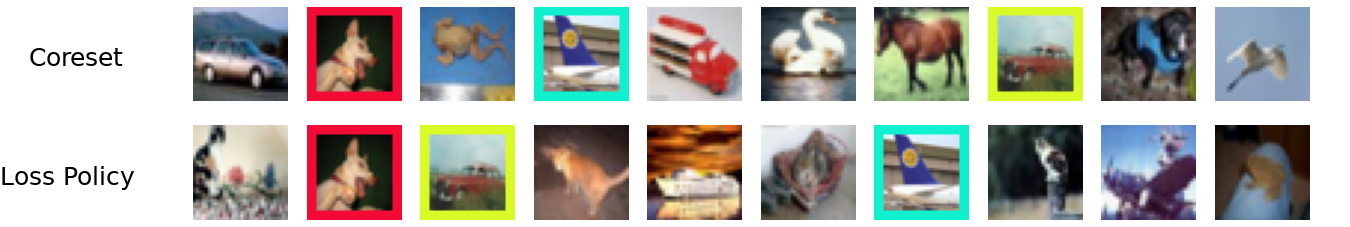}
    \caption{Qualitative evaluation of coreset and max-loss points.}
    \label{fig:coreset_vs_maxloss_points}
\end{figure}

\begin{figure}[!t]
\centering
\subfigure[Loss intersection \label{fig:loss_intersect}]{
	\includegraphics[width=.305\textwidth,trim=10mms 0 0mm 10mm]{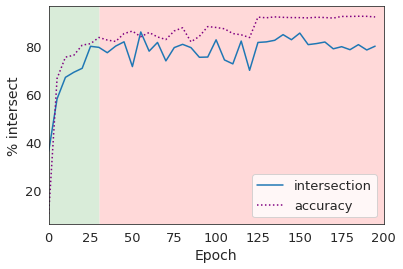}
    }\hspace{0mm}
\subfigure[Improvement \label{fig:coreset_loss_improv}]{    
    \includegraphics[width=.321\textwidth]{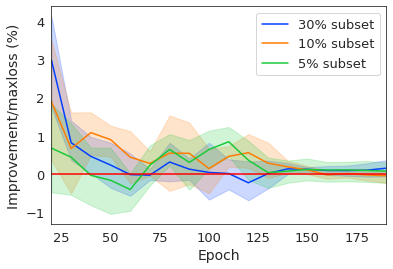}
    }\hspace{0mm}
\subfigure[Coresets vs random \label{fig:coreset_train_vs_not_selected}]{    
    \includegraphics[width=.32\textwidth]{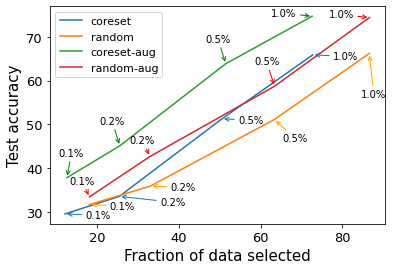}
    }
\vspace{-4mm}    
    \caption{Training ResNet20 on full data and augmented coresets extracted from CIFAR10. (a) Intersection between elements of coresets of size 30\% and maximum loss subsets of the same size. The intersection increases after the initial phase of training, (b) Accuracy improvement for training on full data and augmented coresets over training on full data and  max-loss augmentation. (c) Accuracy vs. fraction of data selected for augmentation during training Resnet20 on CIFAR10.  %
    }%
\vspace{-4mm}
\end{figure}

\subsection{Intersection of max-loss policy and coresets}
\label{app:intersection-maxloss-coreset}
Figure $\ref{fig:loss_intersect}$ depicts the increase in intersection between max-loss subsets and coresets over time. In addition, we also aggregate $30\%$ subsets selected every $R=20$ epochs using both approaches over the entire training process to compute intersection between the top $N$ selected data points. Our plots in Figure \ref{fig:coreset-maxloss-topn-intersect} suggest that a similar pattern holds in this setting. We also qualitatively visualize this in Figure \ref{fig:coreset_vs_maxloss_points}.

\end{document}